\documentclass{article} 

\usepackage{nips15submit_e,times}
\usepackage{amsmath}
\usepackage{amsfonts}
\usepackage{mathrsfs}
\usepackage[disable]{todonotes}
\usepackage{todonotes}
\usepackage{amsthm}
\usepackage[normalem]{ulem}

\usepackage{hyperref}

\presetkeys{todonotes}{inline}{}

\newtheorem{theorem}{Theorem}
\newtheorem{lemma}[theorem]{Lemma}

\newtheorem{definition}[theorem]{Definition}
\newtheorem{proposition}[theorem]{Proposition}

\usepackage{paralist}

\usepackage{times}
\usepackage{graphicx} 
\usepackage{subfigure} 

\usepackage{algorithm}
\usepackage{algorithmic}

\title{Supervised Learning for Dynamical System Learning}

\author{
Ahmed Hefny\\
Carnegie Mellon University\\
Pittsburgh, PA 15213 \\
\texttt{ahefny@cs.cmu.edu} \\
\And
Carlton Downey\\
Carnegie Mellon University\\
Pittsburgh, PA 15213 \\
\texttt{cmdowney@cs.cmu.edu} 
\And
Geoffrey J. Gordon\\
Carnegie Mellon University\\
Pittsburgh, PA 15213 \\
\texttt{ggordon@cs.cmu.edu} \\
}

\nipsfinalcopy 

\begin{document}

\maketitle

\begin{abstract}
  Recently there has been substantial interest in spectral methods for learning dynamical systems. These methods are popular
  since they often offer a good tradeoff between computational 
  and statistical efficiency. 
  Unfortunately, they can be difficult to use and extend
  in practice: e.g., they can make it difficult to incorporate prior
  information such as sparsity or structure.
  To address this problem, we present a new view of dynamical
  system learning: we show how to learn dynamical
  systems by solving a sequence of
  ordinary supervised learning problems, thereby allowing
  users to incorporate prior
  knowledge via standard techniques such as $L_1$ regularization.
  Many existing spectral methods are special cases of this new framework, using linear regression as the supervised learner.
  We demonstrate the effectiveness of our framework by
  showing examples where nonlinear
  regression or lasso let us learn better state representations than
  plain linear regression does; the correctness of these instances follows directly from our general analysis.
\end{abstract}

\def \prob {\mathrm{Pr}}
\def \E {\mathbb{E}}
\def \Pr {\mathrm{Pr}}
\def \tr {\mathrm{tr}}
\def \efs {\eta_{\delta,N}}
\def \eall {\gamma_{\delta,N}}
\def \ecov {\zeta_{\delta,N}}
\def \bxx {{\bar{x}\bar{x}}}
\def \bxy {{\bar{x}\bar{y}}}
\def \ecovxx {\zeta_{\delta,N}^{\bar{x}\bar{x}}}
\def \ecovyy {\zeta_{\delta,N}^{\bar{y}\bar{y}}}
\def \ecovxy {\zeta_{\delta,N}^{\bar{x}\bar{y}}}
\def \ints {\mathbb{N}}
\def \reals {\mathbb{R}}
\def \stat {{\psi}}
\def \fstat {{\xi}}
\def \pstat {{h}}
\def \stattrain {{\mathbf{\Psi}}}
\def \fstattrain {{\mathbf{\Xi}}}
\def \pstattrain {{\mathbf{H}}}
\def \obstrain {{\mathbf{O}}}

\newcommand{\h}[1]{{\cal #1}}
\newcommand{\kernel}[1]{k_{\cal #1}}
\newcommand{\is}[1]{{\cal I_{#1}}}
\newcommand{\inner}[2]{\langle #1 , #2 \rangle}
\newcommand{\innerh}[3]{{\langle #1 , #2 \rangle}_{\h{#3}}}
\newcommand{\lambdax}[1]{\lambda_{x #1}}
\newcommand{\lambday}[1]{\lambda_{y #1}}
\newcommand{\ux}[1]{\psi_{x #1}}
\newcommand{\uy}[1]{\psi_{y #1}}
\newcommand{\Cov}[1]{{\Sigma_{\Bar{#1} \Bar{#1}}}}
\newcommand{\CCov}[2]{{\Sigma_{\Bar{#1} \Bar{#2}}}}
\newcommand{\ACov}[1]{{\hat{\Sigma}_{\Bar{#1} \Bar{#1}}}}
\newcommand{\ACCov}[2]{{\hat{\Sigma}_{\Bar{#1} \Bar{#2}}}}
\newcommand{\AACov}[1]{{\hat{\Sigma}_{\hat{#1} \hat{#1}}}}
\newcommand{\AACCov}[2]{{\hat{\Sigma}_{\hat{#1} \hat{#2}}}}
\newcommand{\cnorm}[1]{\| #1 \|_{XY}}

\newcommand{\srange}[1]{{\cal R}(\Cov{#1})}
\newcommand{\orange}[1]{{\cal R}^\perp(\Cov{#1})}
\newcommand{\scrange}[1]{\overline{\srange{#1}}}

\section{Introduction}


Likelihood-based approaches to learning dynamical systems, such as
EM~\cite{baum70} and MCMC~\cite{Gilks96}, can be slow and suffer from
local optima. This difficulty has resulted in the development of
so-called ``spectral algorithms''~\cite{hsu:09:hmm},
which rely on factorization of a matrix of observable
moments; these algorithms are often fast, simple, and globally optimal.

Despite these advantages, spectral algorithms fall short in one
important aspect compared to EM and MCMC: the latter two methods are
meta-algorithms or frameworks that offer a clear template for
developing new instances incorporating various forms of prior
knowledge. For spectral algorithms, by contrast, there is no clear
template to go from a set of probabilistic assumptions to an
algorithm.  In fact, researchers often relax model assumptions to make
the algorithm design process easier, potentially discarding valuable
information in the process.

To address this problem, we propose a new framework for dynamical
system learning, using the idea of
instrumental-variable regression~\cite{pearl:00,stock:11:econ} to
transform dynamical system learning to a sequence of ordinary
supervised learning problems.  This transformation allows us to apply
the rich literature on supervised learning to incorporate many types
of prior knowledge.  Our new methods
subsume a variety of existing spectral algorithms as special cases.



The remainder of this paper is organized as follows: first we
formulate the new learning framework (Sec.~\ref{sec:ds}). We then
provide theoretical guarantees for the proposed methods
(Sec.~\ref{sec:theory}).  Finally, we give two examples of how our
techniques let us rapidly design new and useful dynamical system learning
methods by encoding modeling assumptions (Sec.~\ref{sec:exp}).

\section{A framework for spectral algorithms}
\label{sec:ds}

\begin{figure}
\centering
\begin{minipage}{0.45\textwidth}
\centering
\includegraphics[clip=true,width=2in,trim=0 4cm 10cm 0]{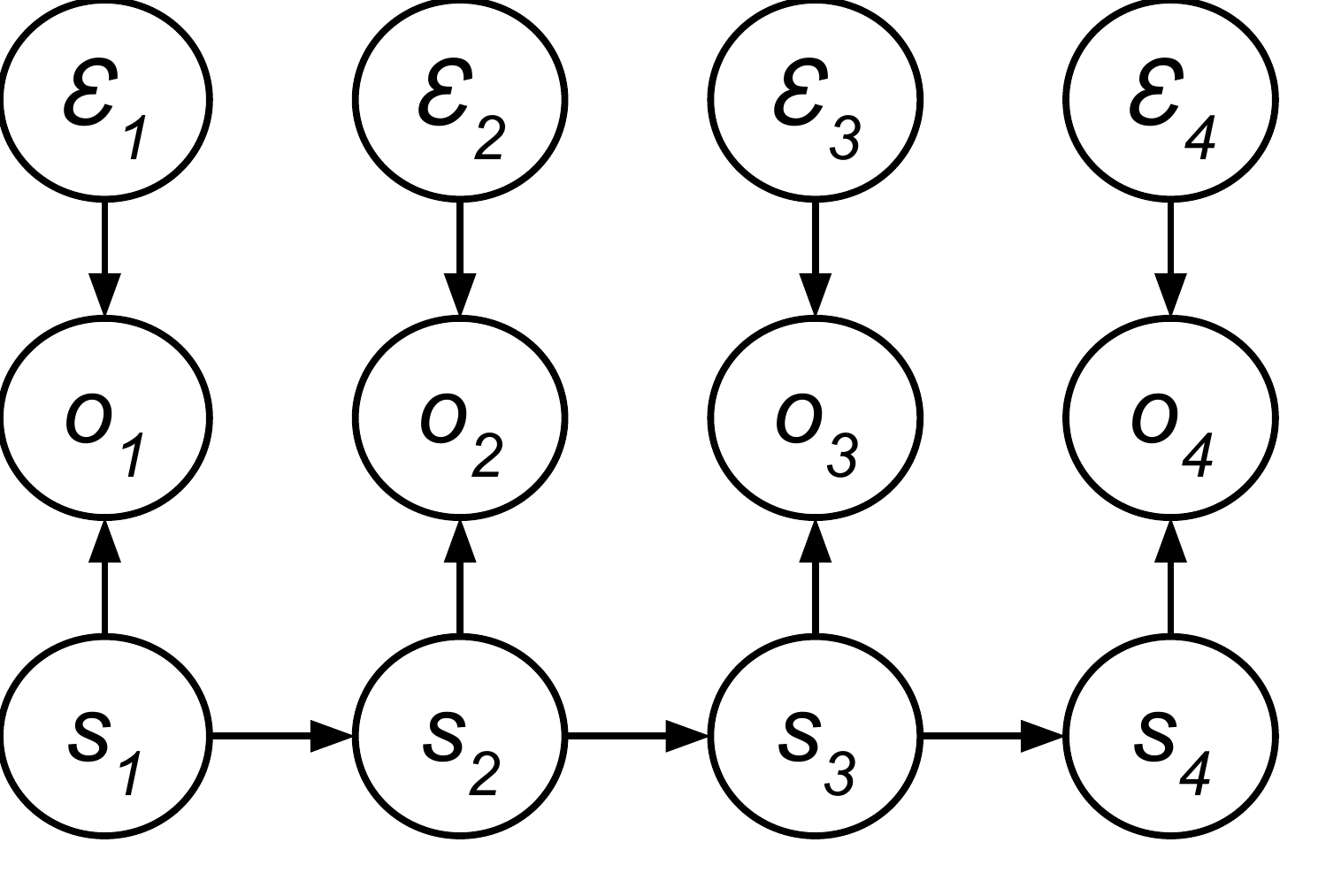}
\caption{A latent-state dynamical system. Observation $o_t$ is determined by latent state $s_t$ and noise $\epsilon_t$.}
\label{fig:ds}
\end{minipage}
\begin{minipage}{0.05\textwidth}
\quad
\end{minipage}
\begin{minipage}{0.45\textwidth}
\centering
\includegraphics[scale=0.35,clip=true,trim=4cm 5.5cm 10cm 1cm]{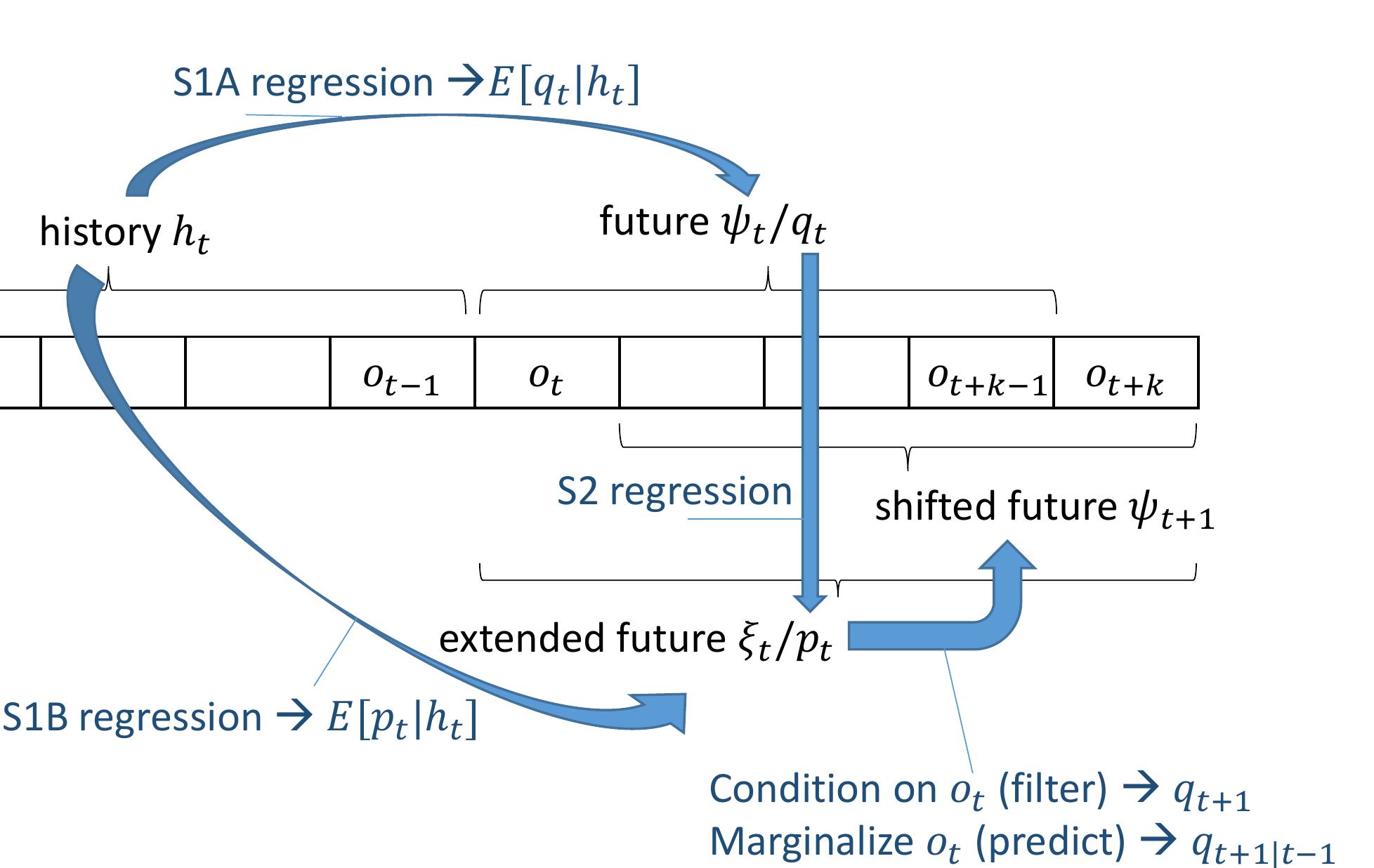}
\caption{Learning and applying a dynamical system with instrumental
  regression.  The predictions from S1 provide training data to S2. At
  test time, we filter or predict using the weights from S2.}
\label{fig:flow}
\end{minipage}
\end{figure}


A dynamical system is a stochastic process (i.e., a distribution over sequences
of observations) such that, at any time, the distribution of
future observations is fully determined by a vector $s_t$ called the \emph{latent state}.
The process is specified by three distributions: the initial state
distribution $P(s_1)$, the state transition distribution
$P(s_{t+1}\mid s_t)$, and the observation distribution $P(o_t\mid
s_t)$.  For later use, we write the observation $o_t$ as a function of
the state $s_t$ and random noise $\epsilon_t$,
as shown in Figure $\ref{fig:ds}$.

Given a dynamical system, one of the fundamental tasks is to perform inference, where we predict future observations given a history of observations.
Typically this is accomplished by maintaining a distribution or
\emph{belief} over states $b_{t\mid t-1} = P(s_t \mid
o_{1:t-1})$ where $o_{1:t-1}$ denotes the first $t-1$
observations. $b_{t\mid t-1}$ represents both our knowledge and our
uncertainty about the true state of the system. 
Two core inference tasks are \emph{filtering} and
\emph{prediction}.\footnote{There are other forms of inference in
  addition to filtering and prediction, such as smoothing and
  likelihood evaluation, but they are outside the scope of this
  paper.}  In filtering, given the current belief $b_t= b_{t\mid t-1}$ and
a new observation $o_t$, we calculate an updated belief $b_{t+1}=b_{t+1\mid
  t}$ that incorporates $o_t$. In prediction, we project our belief into the future: given a
belief $b_{t\mid t-1}$ we estimate $b_{t+k \mid t-1} = P(s_{t+k} \mid
o_{1:t-1})$ for some $k > 0$ (without incorporating any intervening observations).


The typical approach for learning a dynamical system is to explicitly
learn the initial, transition, and observation distributions by
maximum likelihood.  Spectral algorithms offer an alternate approach
to learning: they instead use the method of moments to set up a system
of equations that can be solved in closed form to recover estimates of the
desired parameters.  In this process, they typically factorize a
matrix or tensor of observed moments---hence the name ``spectral.''

Spectral algorithms
often (but not always~\cite{anandkumar2014tensor}) avoid explicitly
estimating the latent state or the initial, transition, or observation
distributions; instead they recover \emph{observable operators} that can
be used to perform filtering and prediction directly.  To do so, they use an
observable representation: instead of maintaining a belief $b_t$ over
states $s_t$, they maintain the expected value of a sufficient
statistic of future observations.  Such a representation is often
called a \emph{(transformed) predictive state}~\cite{citeulike:3980163}.

In more detail, we define
$q_t = q_{t\mid t-1} = \E[\stat_t \mid o_{1:t-1}]$, where
$\stat_t = \stat(o_{t:t+k-1})$ is a vector of \emph{future
  features}. The features are chosen such that $q_t$ determines the
distribution of future observations
$P(o_{t:t+k-1} \mid o_{1:t-1})$.\footnote{For convenience we assume
  that the system is $k$-observable: that is, the distribution of all
  future observations is determined by the distribution of the next
  $k$ observations.  (Note: not by the next $k$ observations themselves.)
  At the cost of additional notation, this restriction could easily be
  lifted.}
Filtering then becomes the process of mapping a predictive state $q_t$
to $q_{t+1}$ conditioned on $o_t$, while prediction maps a predictive
state $q_t=q_{t\mid t-1}$ to
$q_{t+k\mid t-1} = \E[\stat_{t+k} \mid o_{1:t-1}]$ without intervening
observations.

A typical way to derive a spectral method is to select a set of
moments involving $\stat_t$, work out the expected values of these
moments in terms of the observable operators, then invert this
relationship to get an equation for the observable operators in terms
of the moments.  We can then plug in an empirical estimate of the
moments to compute estimates of the observable operators.

While effective, this approach can be statistically inefficient (the
goal of being able to solve for the observable operators is in
conflict with the goal of maximizing statistical efficiency) and can
make it difficult to incorporate prior information (each new source of
information leads to new moments and a different and
possibly harder set of equations to solve).  To address these
problems, we show that we can instead learn the observable operators
by solving three supervised learning problems.

The main idea is that, just as we can represent a belief about a
latent state $s_t$ as the conditional expectation of a vector of
observable statistics, we can also represent any other distributions
needed for prediction and filtering via their own vectors of
observable statistics.  Given such a representation, we can learn to
filter and predict by learning how to map these vectors to one
another.  

In particular, the key intermediate quantity for filtering is the
``extended and marginalized'' belief
$P(o_t,s_{t+1}\mid o_{1:t-1})$---or equivalently
$P(o_{t:t+k}\mid o_{1:t-1})$.  We represent this distribution via a
vector $\fstat_t = \fstat(o_{t:t+k})$ of \emph{features of the
  extended future}. The features are chosen such that the
\emph{extended state} $p_t=\E[\fstat_t \mid o_{1:t-1}]$ determines
$P(o_{t:t+k}\mid o_{1:t-1})$.  Given $P(o_{t:t+k}\mid o_{1:t-1})$,
filtering and prediction reduce respectively to conditioning on and
marginalizing over $o_t$.

In many models (including Hidden Markov Models (HMMs) and
Kalman filters), the extended state $p_t$ is linearly related to the
predictive state $q_t$---a property we exploit for our
framework.  That is,
\begin{align}
p_t = W q_t
\label{eq:moment}
\end{align}
for some linear operator $W$.  For example, in a discrete system $\stat_t$ can be an indicator vector representing the joint assignment of the next $k$ observations, and $\fstat_t$ can be an indicator vector for the next $k+1$ observations. The matrix $W$ is then the conditional probability table $P(o_{t:t+k}\mid o_{t:t+k-1})$. 


Our goal, therefore, is to learn this mapping $W$.  Na\"ively, we
might try to use linear regression for this purpose, substituting
samples of $\stat_t$ and $\fstat_t$ in place of $q_t$ and $p_t$ since
we cannot observe $q_t$ or $p_t$ directly.  Unfortunately, due to the
overlap between observation windows, the noise terms on $\stat_t$ and
$\fstat_t$ are correlated.  So, na\"ive linear regression will give a
biased estimate of $W$.



To counteract this bias, we employ instrumental regression \cite{pearl:00,stock:11:econ}.
Instrumental regression uses \emph{instrumental variables} that are correlated with the input $q_t$
but not with the noise $\epsilon_{t:t+k}$. This property provides a criterion
to denoise the inputs and outputs of the original regression problem:
we remove that part of the input/output that is not correlated with the instrumental variables.
In our case, since past observations $o_{1:t-1}$ do not overlap
with future or extended future windows, they are not correlated with the noise $\epsilon_{t:t+k+1}$, as can be seen in Figure \ref{fig:ds}.
Therefore, we can use \emph{history features} $\pstat_t = \pstat(o_{1:t-1})$ as instrumental variables.


In more detail, by taking the expectation of \eqref{eq:moment} given
$h_t$, we obtain an instrument-based moment condition: for all $t$,
\begin{align}
\nonumber \E [p_t \mid h_t] & = \E [W q_t \mid h_t] \\
\nonumber \E [ \E [\fstat_t \mid o_{1:t-1}] \mid h_t] & = W \E [ \E [\stat_t \mid o_{1:t-1}] \mid h_t]\\
\E [\fstat_t \mid h_t] & = W \E [\stat_t \mid h_t]
\label{eq:ins_moment}
\end{align}
Assuming that there are enough independent dimensions in $h_t$ that
are correlated with $q_t$, we maintain the rank of the moment
condition when moving from \eqref{eq:moment} to \eqref{eq:ins_moment},
and we can recover $W$ by least squares regression if we can compute $\E [\stat_t \mid h_t]$ and
$\E [\fstat_t \mid h_t]$ for sufficiently many examples $t$.

\todo{GG: I removed the footnote about it not being necessary to
  denoise both input and output; it's interesting but perhaps belongs
  in some discussion in an appendix.}

Fortunately, conditional expectations such as $\E [\stat_t \mid h_t]$
are exactly what supervised learning algorithms are designed to
compute.  So, we arrive at our learning framework: we first use
supervised learning to estimate $\E [\stat_t \mid h_t]$ and
$\E [\fstat_t \mid h_t]$, effectively \emph{denoising} the training
examples,
and then use these estimates to compute $W$ by finding the least
squares solution to \eqref{eq:ins_moment}.


In summary, learning and inference of a dynamical system through instrumental regression can be described as follows:

\begin{compactitem}
\item \textbf{Model Specification:} Pick features of history $\pstat_t = \pstat(o_{1:t-1})$, future $\stat_t = \stat(o_{t:t+k-1})$
and extended future $\fstat_t = \fstat(o_{t:t+k})$. 
$\stat_t$ must be a sufficient statistic for $\mathbb{P}(o_{t:t+k-1} \mid o_{1:t-1})$.
$\fstat_t$ must satisfy 
\begin{itemize}
\item $\E[\stat_{t+1} \mid o_{1:t-1}] = f_{\rm predict}(\E[\fstat_t
  \mid o_{1:t-1}])$ for a known function $f_{\rm predict}$.
\item $\E[\stat_{t+1} \mid o_{1:t}] = f_{\rm filter}(\E[\fstat_t
 \mid o_{1:t-1}], o_t)$ for a known function $f_{\rm filter}$.
\end{itemize}
\item \textbf{S1A (Stage 1A) Regression:} Learn a (possibly non-linear) regression model to estimate $\bar{\stat}_t = \E[\stat_t \mid h_t]$.
The training data for this model are $(\pstat_t, \stat_t)$
across time steps $t$.%
\footnote{Our bounds assume that the training time steps $t$ are
  sufficiently
  spaced for the underlying process to mix, but in practice, the error
  will only get smaller if we consider all time steps $t$.} 
\item \textbf{S1B Regression:} Learn a (possibly non-linear) regression model to estimate $\bar{\fstat}_t = \E[\fstat_t \mid h_t]$.
The training data for this model are $(\pstat_t, \fstat_t)$
across time steps $t$.
\item \textbf{S2 Regression:} Use the feature expectations estimated
  in S1A and S1B to train a model to predict
  $\bar{\fstat}_t = W \bar{\stat}_t$, where $W$ is a linear operator.
  The training data for this model are estimates of $(\bar{\stat}_t,
  \bar{\fstat}_t)$ obtained from S1A and S1B across time steps $t$.
\item \textbf{Initial State Estimation:} Estimate an initial state $q_{1} = \E[\stat_1]$ by averaging $\stat_1$
across several example realizations of our time
series.%
\footnote{Assuming ergodicity, we
  can set the initial state to be the empirical average vector of
  future features in a single long sequence, $\frac{1}{T}\sum_{t=1}^T\stat_t$.} 
\item \textbf{Inference:} Starting from the initial state $q_{1}$, we can maintain 
the predictive state $q_t = \E[\stat_t \mid o_{1:t-1}]$ through filtering:
given $q_t$ we compute $p_t = \E[\fstat_{t} \mid o_{1:t-1}] = W q_t$.
Then, given the observation $o_t$, we can compute $q_{t+1} = f_{\rm filter}(p_t, o_t)$. 
Or, in the absence of $o_t$, we can predict the next state
$q_{t+1\mid t-1} = f_{\rm predict}(p_t)$.
Finally, by definition, the predictive state $q_t$ is sufficient to
compute $\mathbb{P}(o_{t:t+k-1} \mid o_{1:t-1})$.\footnote{It might
  seem reasonable to learn $q_{t+1} = f_{\rm combined}(q_{t}, o_{t})$ directly,
  thereby avoiding the need to separately estimate $p_t$ and condition
  on $o_t$.  Unfortunately, 
  $f_{\rm combined}$ is nonlinear for common models such as HMMs.}
\end{compactitem}
The process of learning and inference is depicted in Figure \ref{fig:flow}.
Modeling assumptions are reflected in the choice of the statistics $\stat$, $\fstat$ and $\pstat$
as well as the regression models in stages S1A and S1B\@. 
Table \ref{tbl:examples} demonstrates that we can recover existing spectral algorithms for dynamical system learning
using linear S1 regression.  In addition to providing a unifying view
of some successful learning algorithms, the
new framework also paves the way for extending these algorithms in a
theoretically justified manner, as we demonstrate in the experiments below.


\begin{table}
\scriptsize
\centering
\begin{tabular}{|p{2cm}|p{4cm}|p{2.5cm}|p{4cm}|}
\hline
 Model/Algorithm & future features $\stat_t$ & extended future features $\fstat_t$ & $f_{\rm filter}$    \\
 \hline
 Spectral Algorithm for HMM \cite{hsu:09:hmm} & $U^\top e_{o_t}$ where $e_{o_t}$ is an indicator vector and $U$ spans the range of $q_t$ (typically the top $m$ left singular vectors
 of the joint probability table $P(o_{t+1},o_{t})$) & $U^\top e_{o_{t+1}} \otimes e_{o_t}$ & Estimate a state normalizer from S1A
 output states. \\
 \hline
 SSID for Kalman filters (time dependent gain)  & $x_t$ and $x_t \otimes x_t$, where $x_t = U^\top o_{t:t+k-1}$ for a matrix $U$ that spans the range of $q_t$ (typically the top $m$ left singular vectors
 of the covariance matrix
 $\mathrm{Cov}(o_{t:t+k-1},o_{t-k:t-1})$) & $y_t$ and $y_t \otimes y_t$, where $y_t$ is formed by stacking $U^\top o_{t+1:t+k}$ and $o_t$.
 & $p_t$ specifies a Gaussian distribution where conditioning on $o_t$ is straightforward. \\
 \hline
 SSID for stable Kalman filters (constant gain) & $U^\top o_{t:t+k-1}$ ($U$ obtained as above) & $o_t$ and $U^\top o_{t+1:t+k}$ & Estimate steady-state covariance by solving Riccati equation \cite{vanoverschee:96}. $p_t$ together with the steady-state covariance specify a Gaussian distribution where conditioning on $o_t$ is straightforward. \\
 \hline
 Uncontrolled HSE-PSR \cite{boots:13:hsepsr} & Evaluation functional $k_s(o_{t:t+k-1},.)$ for a characteristic kernel $k_s$ 
 & $k_o(o_t,.) \otimes k_o(o_t,.)$ and 
 $\stat_{t+1} \otimes k_o(o_t,.)$ & Kernel Bayes rule \cite{fukumizu:13:kbr}. \\
 \hline
\end{tabular}
\caption{Examples of existing spectral algorithms reformulated as two-stage instrument regression with linear S1 regression. Here $o_{t_1:t_2}$ is a vector formed by stacking observations $o_{t_1}$ through $o_{t_2}$ and $\otimes$ denotes the outer product.
Details and derivations can be found in the supplementary material.}
\label{tbl:examples}
\end{table}

\section{Related Work}
\label{sec:related}
This work extends predictive state learning algorithms for dynamical systems, which include
spectral algorithms for Kalman filters \cite{boots:12:thesis}, Hidden
Markov Models \cite{hsu:09:hmm,siddiqi:10:rrhmm},  Predictive State
Representations (PSRs) \cite{boots:11:plan,boots:11b:onlinepsr} and
Weighted Automata \cite{balle:14}.  It also extends kernel variants
such as \cite{boots:13:hsepsr}, which builds on
\cite{song:10:hsehmm}.  
All of the above work effectively uses linear regression or linear
ridge regression (although not always in an obvious way).  


One common aspect of predictive state learning algorithms is that they exploit the covariance structure between 
future and past observation sequences to obtain an unbiased observable state representation.
Boots and Gordon \cite{boots:12:2man} note the connection between this
covariance and (linear) instrumental regression in the context of the HSE-HMM\@. We use this connection
to build a general framework for dynamical system learning where the
state space can be identified using arbitrary (possibly nonlinear) supervised learning methods.
This generalization lets us incorporate
prior knowledge to learn compact or regularized models; our
experiments demonstrate that this flexibility lets us take better
advantage of limited data.

Reducing the problem of learning dynamical systems with latent state to supervised learning 
bears similarity to Langford et al.'s sufficient posterior representation (SPR) \cite{langford:09},
which encodes the state by the sufficient statistics of the conditional distribution of the next observation
and represents system dynamics by three vector-valued functions that are estimated using supervised learning approaches.
While SPR allows all of these functions to be non-linear,
it involves a rather complicated training procedure involving multiple iterations of model refinement and 
model averaging, whereas our framework only requires solving three
regression problems in sequence. 
In addition, the theoretical analysis of \cite{langford:09} only
establishes the consistency of SPR learning assuming that 
all regression steps are solved perfectly. Our work, on the other hand, establishes convergence rates based on the performance 
of S1 regression.

\section{Theoretical Analysis}
\label{sec:theory}
In this section we present error bounds for two-stage instrumental
regression.  These bounds hold
regardless of the particular S1 regression method used,
assuming that the S1 predictions converge to the true conditional expectations.
The bounds imply that our overall method is consistent.


Let $(x_t, y_t, z_t)\in(\h{X},\h{Y},\h{Z})$ be i.i.d.\ triplets of
input, output, and instrumental variables.
(Lack of independence will result in slower
convergence in proportion to the mixing time of our process.)
Let $\bar{x}_t$ and $\bar{y}_t$ denote 
$\E[x_t \mid  z_t]$ and $\E[y_t \mid  z_t]$. And, let $\hat{x}_t$ and $\hat{y}_t$ denote 
$\hat{\E}[x_t \mid  z_t]$ and $\hat{\E}[y_t \mid  z_t]$ as estimated by the S1A
and S1B regression steps.
Here $\bar{x}_t, \hat{x}_t \in \h{X}$ and $\bar{y}_t, \hat{y}_t \in \h{Y}$.

We want to analyze the convergence of the output of S2
regression---that is, of the weights $W$ given by ridge regression
between S1A outputs and S1B outputs:
\begin{align}
\hat{W}_\lambda = \left(\sum_{t=1}^T \hat{y}_t \otimes \hat{x}_t \right) \left(\sum_{t=1}^T \hat{x}_t \otimes \hat{x}_t + \lambda I_\h{X}\right)^{-1}
\label{eq:what}
\end{align}
Here $\otimes$ denotes tensor (outer) product, and 
$\lambda > 0$ is a regularization parameter that ensures the invertibility of the estimated covariance.

Before we state our main theorem we need to quantify the quality of S1 regression
in a way that is independent of the S1 functional form.  To do so, we
place a bound on the S1 error, and assume that this bound converges to
zero: given the definition below, 
for each fixed $\delta$,
$\lim_{N \to \infty} \efs = 0$.


\begin{definition}[S1 Regression Bound]
For any $\delta > 0$ and $N \in \ints^+$, 
the S1 regression bound $\efs > 0$
is a number such that,
with probability at least $(1 - \delta/2)$,
for all $1 \leq t \leq N$:
\begin{align*}
\|\hat{x}_t - \bar{x}_t\|_\h{X} & < \efs \\
\|\hat{y}_t - \bar{y}_t\|_\h{Y} & < \efs 
\end{align*}
\label{thm:efs}
\end{definition}\label{thm:consistent_s1}
%
%
%
\vspace{-6mm} In many applications, $\h{X}$, $\h{Y}$ and $\h{Z}$ will be finite
dimensional real vector spaces: $\mathbb{R}^{d_x}$, $\mathbb{R}^{d_y}$
and $\mathbb{R}^{d_z}$.  However, for generality we state our results
in terms of arbitrary reproducing kernel Hilbert spaces.  In this case
S2 uses kernel ridge regression, leading to methods such as HSE-PSRs.
For this purpose, let $\Cov{x}$ 
and $\Cov{y}$ 
denote the (uncentered) covariance operators of 
$\bar{x}$ and $\bar{y}$ respectively: $
\Cov{x} = \E[\bar{x} \otimes \bar{x}],\ \Cov{y} = \E[\bar{y} \otimes \bar{y}]
$.
And, let $\scrange{x}$ denote the closure of the range of $\Cov{x}$.

\todo{AH: the papers don't talk much about the existence of the covariance operator. 
I think it follows from Reisz representation theorem with additional arguments about the boundedness of the 
covariance, which implies continuity and hence linearity on the completion of tensor products.}

With the above assumptions, Theorem~\ref{thm:main_short} gives a
generic error bound on S2 regression in terms of S1 regression.  If
$\h{X}$ and $\h{Y}$ are finite dimensional and $\Cov{x}$ has full
rank, then using ordinary least squares (i.e., setting $\lambda = 0$)
will give the same bound, but with $\lambda$ in the first two terms
replaced by the minimum eigenvalue of $\Cov{x}$, and the last term dropped.


\begin{theorem}
Assume that $\|\bar{x}\|_{\h{X}}, \|\bar{x}\|_{\h{Y}} < c < \infty$ almost
surely.
Assume $W$ is a Hilbert-Schmidt operator, and let $\hat{W}_\lambda$
be as defined in \eqref{eq:what}.
Then, with probability at least $1 - \delta$,
for each $x_{\rm test} \in \scrange{x}$ s.t.\ $\|x_{\rm test}\|_{\h{X}}
\leq 1$, the error $\|\hat{W}_\lambda {x}_{\rm test} - W {x}_{\rm
  test}\|_\h{Y}$ is bounded by
\todo{AH: The constant depends on $\|\Cov{x}^{-1/2} \bar{x}_{test}\|$. Thus used "for each".
Note though that we don't need the union bound for a finite set of points.
Also, with some upper bound assumption on $\|\Cov{x}^{-1/2} \bar{x}_{test}\|$, we get a uniform result (see the proof of the regularization part)}
\begin{align*}
\quad \underbrace{O\left(\efs \left(\frac{1}{\lambda} + \frac{\sqrt{1 + \sqrt{\frac{\log(1/\delta)}{N}}}}{\lambda^\frac{3}{2}} \right)\right)}_{\rm error\ in\ S1\ regression}
  + \underbrace{O\left(\frac{\log(1/\delta)}{\sqrt{N}} \left(\frac{1}{\lambda} + \frac{1}{\lambda^\frac{3}{2}} \right)\right)}_{\rm error\ from\ finite\ samples} 
  + \underbrace{O\left( \sqrt{\lambda} \right)}_{\rm error\ from\ regularization} 
\end{align*}
\label{thm:main_short}
\end{theorem}
\todo{GG: I'm worried that reviewers will not like how much we defer
  to the supplementary material.  Not sure how to fix that, though\ldots}
\todo{GG: it says we defer a finite sample analysis to the appendix,
  but theorem 2 seems to say something about behavior on finite
  samples?  It would be much better if we didn't have to say this was
  deferred.}
\todo{GG: I don't understand the effect of decomposing $\hat x_{\rm
    test}$.  It looks restrictive that $\epsilon$ has to go to 0.}

\vspace{-5mm} We defer the proof to the supplementary material.  The supplementary
material also provides explicit finite-sample bounds (including
expressions for the constants hidden by $O$-notation), as well as
concrete examples of S1 regression bounds $\efs$ for practical
regression models.


Theorem \ref{thm:main_short} assumes that $x_{\rm test}$ is in
$\scrange{x}$.  For dynamical systems, all valid states satisfy this
property.  However, with finite data, estimation errors may cause the
estimated state $\hat{q}_t$ (i.e., $x_{\rm test}$) to have a non-zero
component in $\orange{x}$.  Lemma~\ref{thm:orth_state_bound} bounds
the effect of such errors: it states that, in a stable system, this
component gets smaller as S1 regression performs better.  The main
limitation of Lemma~\ref{thm:orth_state_bound} is the assumption that
$f_{\rm filter}$ is $L$-Lipchitz, which essentially means that the
model's estimated probability for $o_t$ is bounded below.  There is
no way to guarantee this property in practice; so,
Lemma~\ref{thm:orth_state_bound} provides suggestive evidence rather
than a guarantee that our learned dynamical system will predict well.

\begin{lemma}
For observations $o_{1:T}$, 
let $\hat{q}_t$ be the estimated state given $o_{1:t-1}$.
Let $\tilde{q}_t$ be the projection of $\hat{q}_t$ onto $\scrange{x}$.
Assume $f_{\rm filter}$ is $L$-Lipchitz on $p_t$
when evaluated at $o_t$,
and $f_{\rm filter}(p_t, o_t) \in \scrange{x}$ for any $p_t \in \scrange{y}$.
Given the assumptions of theorem $\ref{thm:main_short}$
and assuming that $\| \hat{q}_t \|_\h{X} \leq R$ for all $1 \leq t \leq T$,
the following holds for all $1 \leq t \leq T$ with probability at least $1 - \delta/2$.
\begin{align*}
\| \epsilon_t \|_\h{X} = \| \hat{q}_t - \tilde{q}_t \|_\h{X} = O \left(\frac{\efs}{\sqrt{\lambda}} \right)
\end{align*}
\vspace{-5mm}
\label{thm:orth_state_bound}
\end{lemma}
\todo{Note that the constant depends on $R = \max_t \| \hat{Q_t} \|_\h{X}$, which in turn 
depends on the trained model. So I am not sure whether this statement makes sense.}
Since $\hat{W}_\lambda$ is bounded, the prediction error due to $\epsilon_t$ diminishes at the same rate as $\|\epsilon_t\|_{\cal X}$.



\section{Experiments and Results}
\label{sec:exp}
We now demonstrate examples of tweaking the S1 regression to gain advantage.
In the first experiment we show that nonlinear regression can be used to reduce the number of parameters needed in S1, thereby improving statistical performance for learning an HMM\@. In the second experiment we show that we can encode prior knowledge as regularization.

\subsection{Learning A Knowledge Tracing Model}
\label{sec:exp-kt} 

In this experiment we attempt to model and predict the performance of students learning from an interactive computer-based tutor.  We use the Bayesian knowledge tracing (BKT) model \cite{corbett:95:kt}, which is essentially a 2-state HMM: the state $s_t$ represents whether a student has learned a knowledge component (KC), and the observation $o_t$ represents the success/failure of solving the $t^{\rm th}$ question in a sequence of questions that cover this KC\@.  
Figure \ref{fig:bkt} summarizes the model.
The events denoted by guessing, slipping, learning and forgetting typically have relatively low probabilities.


%
%

\begin{figure}
\centering
\includegraphics[height=3cm,width=6cm,clip=true,trim=0cm 5.5cm 0cm 0cm]{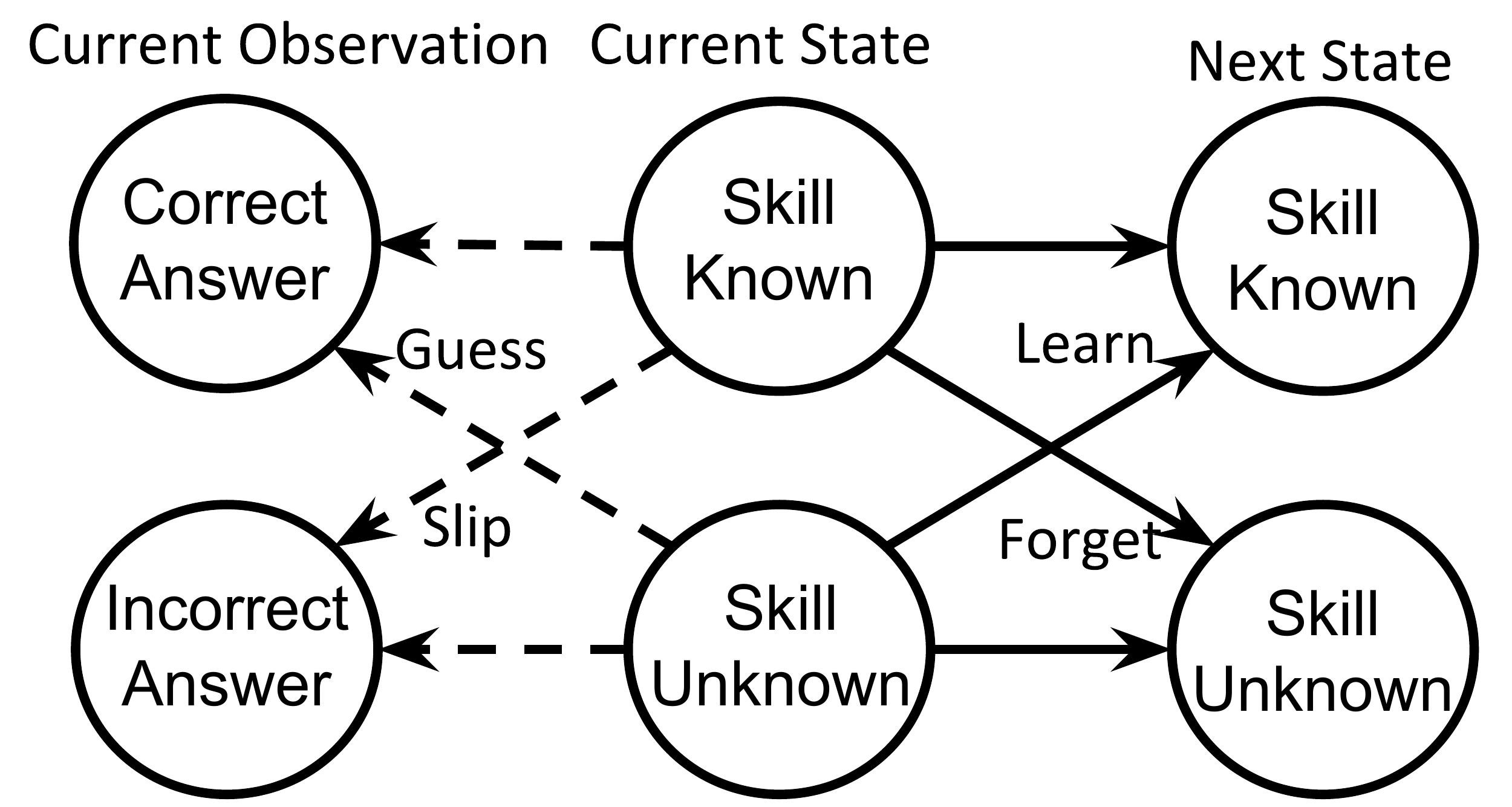}
\caption{Transitions and observations in BKT\@. Each node represents a possible \emph{value}
of the state or observation.
Solid arrows represent transitions while dashed arrows represent observations.
}
\label{fig:bkt}
\end{figure}

\subsubsection{Data Description}
We evaluate the model using the ``Geometry Area (1996-97)'' data available from DataShop \cite{koedinger:10:data}.
This data was generated by students learning introductory geometry, and contains attempts by 59 students in 12 knowledge components. As is typical for BKT, we consider a student's attempt at a question to be correct iff the student entered the correct answer on the first try, without requesting any hints from the help system.
Each training sequence consists of a sequence of first attempts for a student/KC pair. 
We discard sequences of length less than 5, resulting in a total of 325
sequences.  


\todo{GG: do we pad sequences at the beginning with special ``before
  the start of time'' observations?  This lets us take advantage of
  shorter sequences as well (where the history window extends before
  the beginning of the sequence), as well as the first few steps of
  longer sequences, increasing the available data.
	AH:Yes we do. I edited the text to reflect that. The reason for discarding sequences is for testing no training; to make
	sure that test sequences have sufficient steps for initial filtering.  
  GG: Interesting---what happens if we include the short sequences at
  test time?  Does it just increase the noise level of predictions?
  (We should be able to predict for them as well, so it would be
  preferable to include them if possible in a future version of the
  experiments.)}


\subsubsection{Models and Evaluation}
Under the (reasonable) assumption that the two states have distinct observation probabilities, this model is 1-observable. Hence we define the predictive state to be the expected next observation, which results in the following statistics: $\stat_t = o_t$ and $\fstat_t = o_t \otimes_k o_{t+1}$,
where $o_t$ is represented by a 2 dimensional indicator vector and  $\otimes_k$ denotes the Kronecker product. 
Given these statistics, the extended state $p_t = \E[\fstat_t \mid  o_{1:t-1}]$ 
is a joint probability table of $o_{t:t+1}$.

We compare three models that differ by history features and S1 regression method:

\textbf{Spec-HMM:}
This baseline uses $h_t = o_{t-1}$ and linear S1 regression, making it equivalent to the spectral HMM method of \cite{hsu:09:hmm}, as detailed in the supplementary material.

\textbf{Feat-HMM:}
This baseline represents $h_t$ by an indicator vector of the joint assignment
of the previous $b$ observations (we set $b$ to 4) and uses linear S1 regression. This is essentially
a feature-based spectral HMM \cite{siddiqi:10:rrhmm}.
It thus incorporates more history information compared to Spec-HMM
at the expense of increasing the number of S1 parameters by $O(2^b)$.

\textbf{LR-HMM:} 
This model represents $h_t$ by a binary vector of length $b$ encoding the previous $b$ observations and uses logistic regression as the S1 model. Thus, it uses the same history information 
as Feat-HMM but reduces the number of parameters to $O(b)$ at the expense of inductive bias.

%
%

We evaluated the above models using 1000 random splits of the 325 sequences into 200 training and 125 testing. For each testing observation
$o_t$ we compute the absolute error between actual and expected value (i.e. $|\delta_{o_t = 1} - \hat{P}(o_t = 1 \mid  o_{1:t-1})|$). 
We report the mean absolute error for each split.
The results are displayed in Figure \ref{fig:results}.\footnote{The differences have similar sign but smaller magnitude if we use RMSE instead of MAE.}
 We see that,
while incorporating more history information increases accuracy (Feat-HMM vs.\ Spec-HMM), being able to incorporate the
same information using a more compact model gives an additional gain in accuracy (LR-HMM vs.\ Feat-HMM).
We also compared the LR-HMM method to an HMM trained using expectation maximization (EM). We found that the LR-HMM model is much
faster to train than EM
while being on par with it in terms of prediction error.\footnote{We used MATLAB's built-in logistic regression and EM functions.}

\begin{figure}[h!]
\centering
\begin{tabular}{cccc}
\includegraphics[scale=0.25]{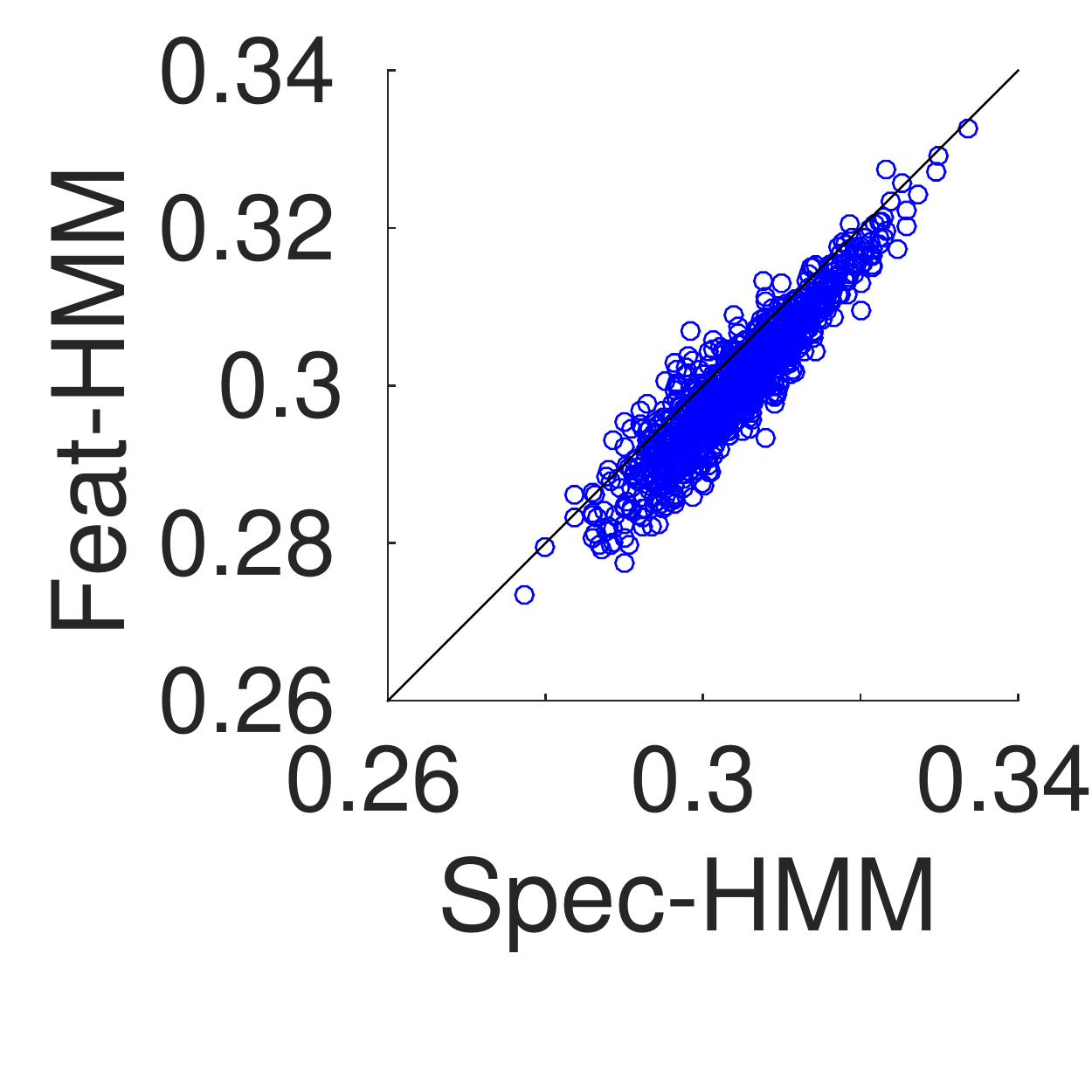}\hspace{-4mm} 
& \includegraphics[scale=0.25]{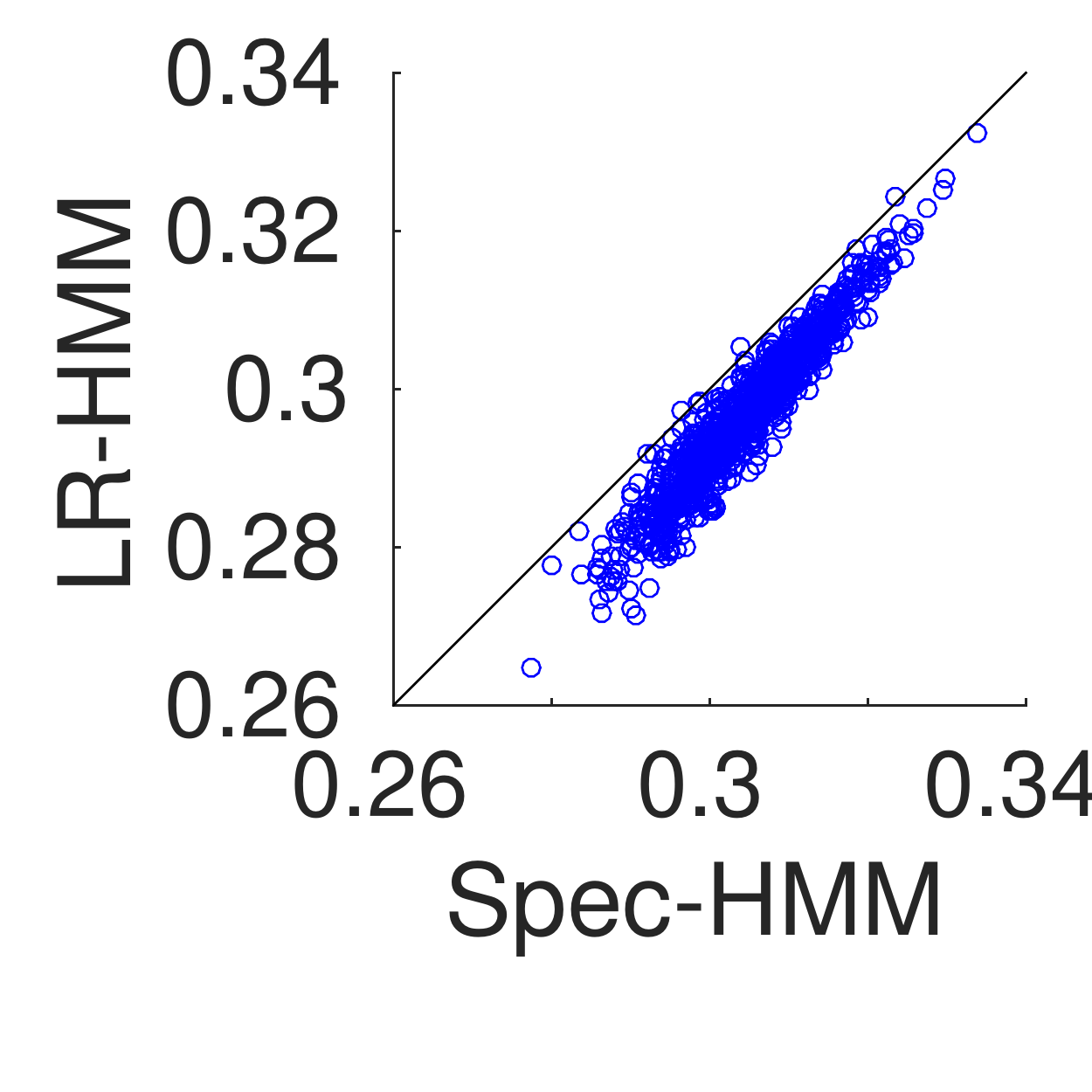}\hspace{-4mm} 
& \includegraphics[scale=0.25]{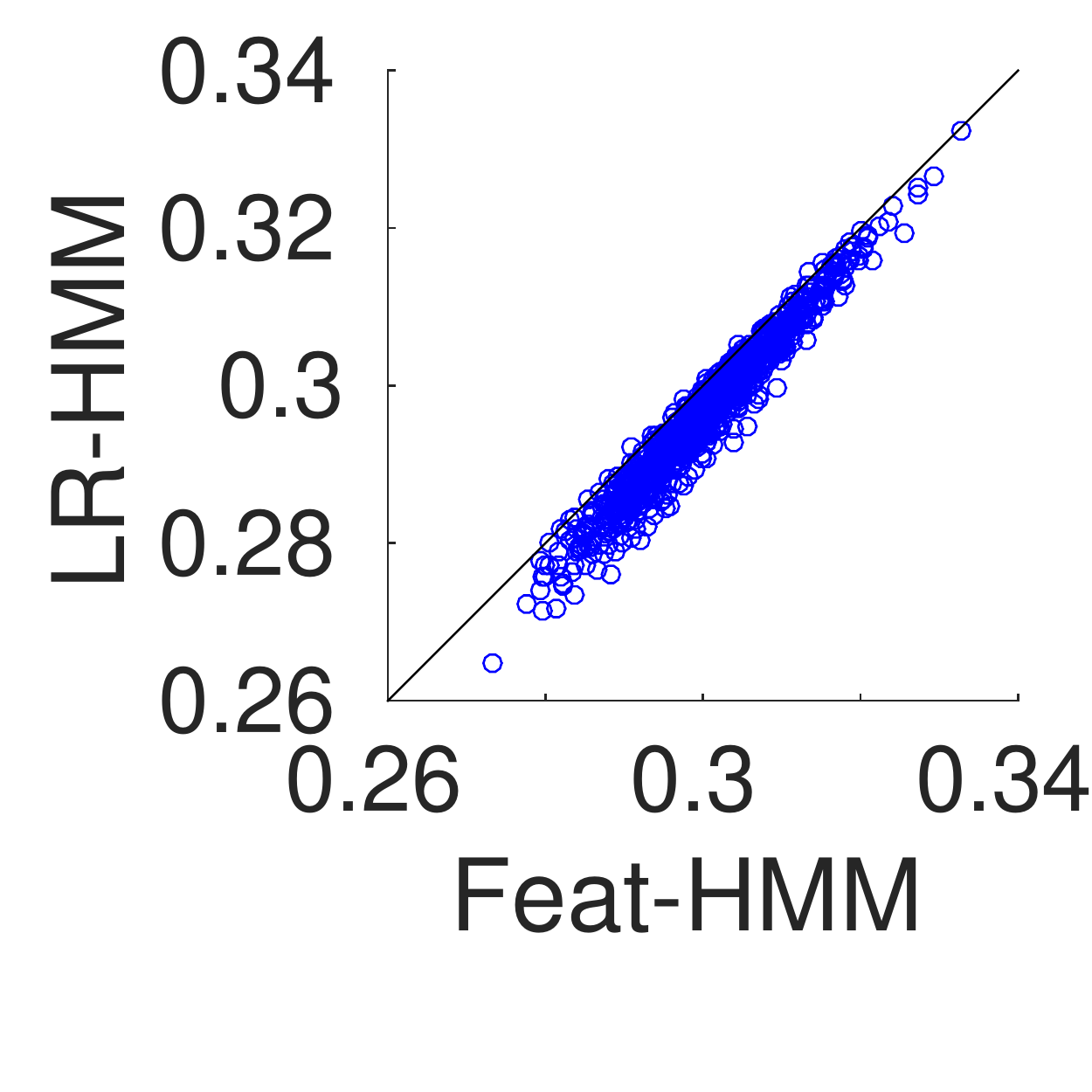}\hspace{-4mm} 
& \includegraphics[scale=0.25]{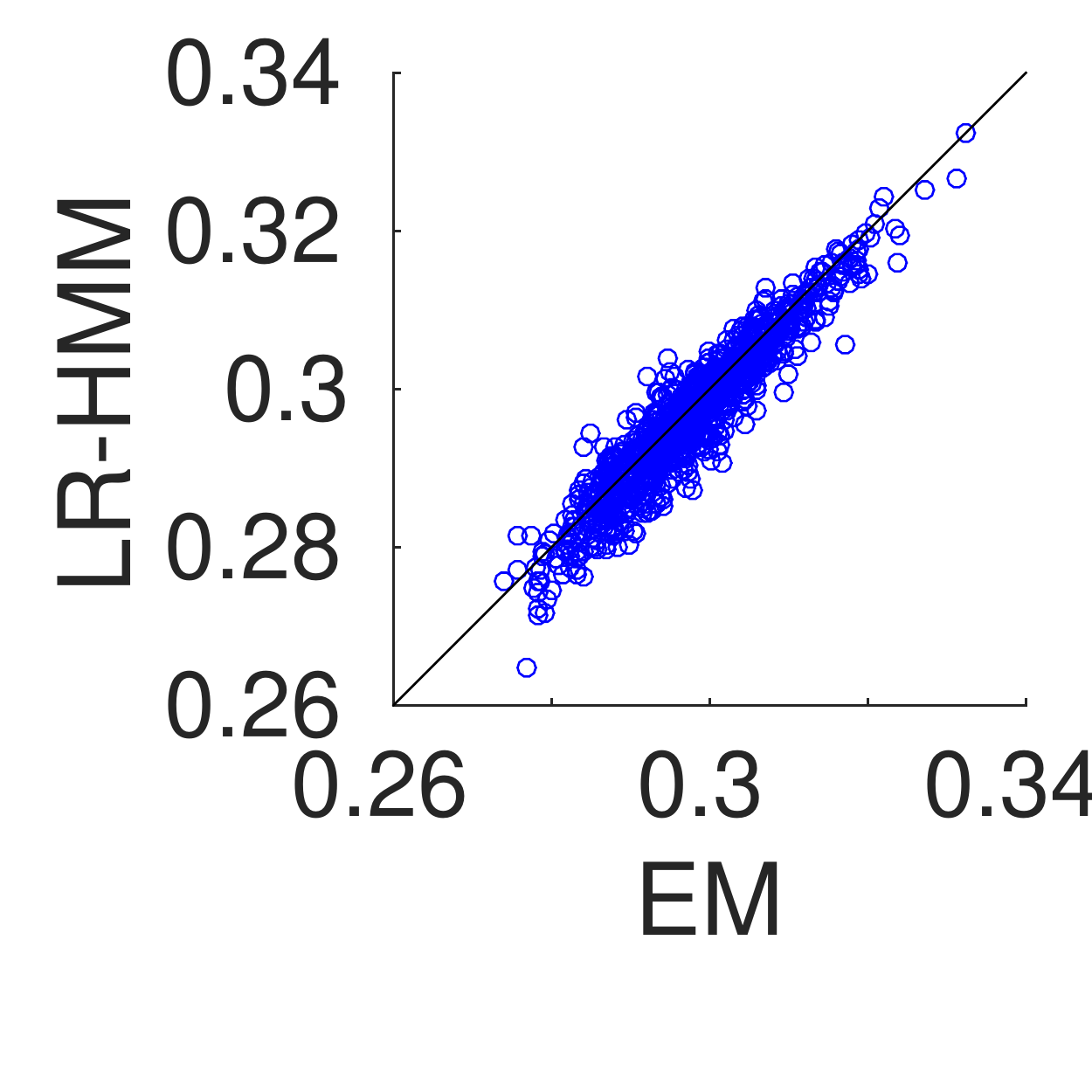}
\end{tabular}\\
\small
\begin{tabular}{|c|c|c|c|c|}
\hline
Model & Spec-HMM & Feat-HMM & LR-HMM & EM\\
\hline 
Training time (relative to Spec-HMM) & 1 & 1.02 & 2.219 & 14.323\\
\hline 
\end{tabular}
\caption{Experimental results: each graph compares the performance of two models
(measured by mean absolute error) on 1000 train/test splits. The black line is $x = y$.
Points below this line indicate that model $y$ is better than model $x$. The table shows training time.
}
\label{fig:results}
\end{figure}
\todo{GG: fonts in fig 5 should be much larger.}
\todo{GG: might want to consider also a model with separate features
  and a linear regression}
\todo{GG: might be able to show more benefit if we predict farther
  into the future: filter for $k$ steps, predict without incorporating
  observations for $k'$ steps, then evaluate RMSE on the next
  observation.}

\subsection{Modeling Independent Subsystems Using Lasso Regression}

Spectral algorithms for Kalman filters typically use the left singular vectors of the covariance between
history and future features as a basis for the state space.  However,
this basis hides any sparsity that might be present in our original basis.
In this experiment, we show that we can instead use lasso (without
dimensionality reduction) as our S1 regression algorithm to discover sparsity.
This is useful, for example, when the system consists of multiple independent subsystems, each of which 
affects a subset of the observation coordinates. 

To test this idea we generate a sequence of 30-dimensional
observations from a Kalman filter.  Observation dimensions 1 through
10 and 11 through 20 are generated from two independent subsystems of
state dimension 5. Dimensions 21-30 are generated from white
noise. Each subsystem's transition and observation matrices have
random Gaussian coordinates, with the transition matrix scaled to have
a maximum eigenvalue of 0.95. States and observations are perturbed by
Gaussian noise with covariance of $0.01 I$ and $1.0 I$ respectively.

We estimate the state space basis using 1000 examples (assuming 1-observability) and compare the singular vectors of the past to future regression matrix to those obtained from the Lasso regression matrix. The result is shown in figure \ref{fig:u}. Clearly, using Lasso as stage 1 regression 
results in a basis that better matches the structure of the underlying system.

\begin{figure}
\centering
\includegraphics[width=2.2cm,height=2.2cm]{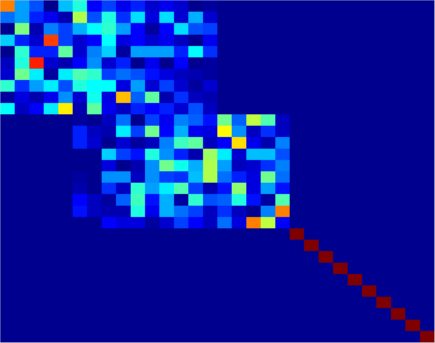}
\includegraphics[width=2.2cm,height=2.2cm]{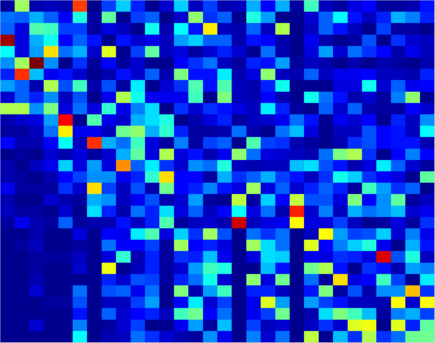}
\includegraphics[width=2.2cm,height=2.2cm]{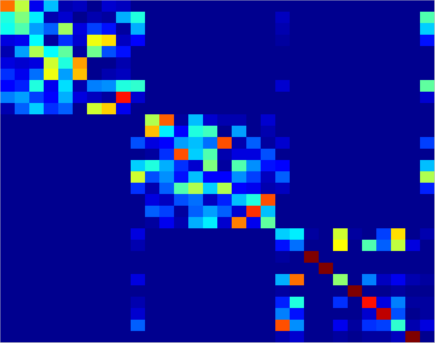}
\includegraphics[width=1cm,height=2.2cm]{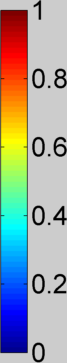}
\caption{Left singular vectors of (left) true linear predictor from $o_{t-1}$ to $o_t$ (i.e. $O T O^+$),
(middle) covariance matrix between $o_{t}$ and $o_{t-1}$ and (right) S1 Sparse regression weights.
Each column corresponds to a singular vector (only absolute values are depicted). Singular vectors are ordered
by their mean coordinate, interpreting absolute values as a probability distribution over coordinates.
}
\label{fig:u}
\end{figure}
\section{Conclusion}
In this work we developed a general framework for dynamical system learning using supervised learning methods.  The framework relies on two key principles: first, we extend the idea of predictive state to include extended state as well, allowing us to represent all of inference in terms of predictions of observable features.  Second, we use past features as instruments in an instrumental regression, denoising state estimates that then serve as training examples to estimate system dynamics.

We have shown that this framework encompasses and provides a unified view of some previous successful dynamical system learning algorithms. 
We have also demostrated that it can be used to extend existing algorithms
to incorporate nonlinearity and regularizers, resulting in better state estimates.
As future work, we would like to apply this framework to leverage additional techniques such as manifold embedding and transfer learning
in stage 1 regression. We would also like to extend the framework to controlled processes.

\nocite{song:09}

\bibliographystyle{unsrt}
\bibliography{ins_regression}

\clearpage
\appendix
\renewcommand\thesection{\Alph{section}}
\numberwithin{equation}{section}
\numberwithin{theorem}{section}
\section{Spectral and HSE Dynamical System Learning as Regression}
In this section we provide examples of mapping some of the successful dynamical system learning algorithms to our framework.
\subsection{HMM}
In this section we show that we can use instrumental regression framework to reproduce 
the spectral learning algorithm for learning HMM \cite{hsu:09:hmm}.
We consider 1-observable models but the argument applies to $k$-observable models.
In this case we use $\stat_t = e_{o_t}$ and $\fstat_t = e_{o_{t:t+1}} = e_{o_t} \otimes_k e_{o_{t+1}}$, where $\otimes_k$ denotes the kronecker product.
Let $P_{i,j} \equiv \E[e_{o_i} \otimes e_{o_j}]$ be the joint probability table of observations $i$ and $j$
and let $\hat{P}_{i,j}$ be its estimate from the data.
We start with the (very restrictive) case where $P_{1,2}$ is invertible.
Given samples of $h_2 = e_{o_1}$, $\stat_2 = e_{o_2}$ and $\fstat_2 = e_{o_{2:3}}$, 
in S1 regression we apply linear regression to learn two matrices $\hat{W}_{2,1}$ and $\hat{W}_{2:3,1}$ such that:
\todo{I also assume $P_{1,1}$ is invertible (i.e. the sequence can start with any observation
with non-zero probability)}
\begin{align}
\hat{\E}[\stat_2 | h_2] & = \hat{\Sigma}_{o_2 o_1}  \hat{\Sigma}_{o_1}^{-1} h_2 = \hat{P}_{2,1} \hat{P}_{1,1}^{-1} h_t \equiv \hat{W}_{2,1} h_2\\
\hat{\E}[\fstat_2 | h_2] & = \hat{\Sigma}_{o_{2:3} o_1}  \hat{\Sigma}_{o_1}^{-1} h_2 = \hat{P}_{2:3,1} \hat{P}_{1,1}^{-1} h_2 \equiv \hat{W}_{2:3,1} h_2,
\end{align}

where $P_{2:3,1} \equiv \E[e_{o_{2:3}} \otimes e_{o_1}]$

In S2 regression, we learn the matrix $\hat{W}$ that gives the least squares solution to the system of equations
\begin{align*}
\hat{\E}[\fstat_2 | h_2] \equiv \hat{W}_{2:3,1} e_{o_1}  = \hat{W} (\hat{W}_{2,1} e_{o_1}) \equiv \hat{W} \hat{\E}[\stat_2 | h_2] \quad , \text{for given samples of $h_2$}
\end{align*}
which gives
\begin{align}
\nonumber \hat{W} & = \hat{W}_{2:3,1} \hat{\E}[e_{o_1} e_{o_1}^\top] \hat{W}_{2,1}^\top
\left(\hat{W}_{2,1} \hat{\E}[e_{o_1} e_{o_1}^\top] \hat{W}_{2,1}^\top \right)^{-1} \\
\nonumber  & = \left(\hat{P}_{2:3,1} \hat{P}_{1,1}^{-1} \hat{P}_{2,1}^\top\right)
\left(\hat{P}_{2,1} \hat{P}_{1,1}^{-1} \hat{P}_{2,1}^\top\right)^{-1} \\
& = \hat{P}_{2:3,1} \left( \hat{P}_{2,1}\right)^{-1}
\end{align}

Having learned the matrix $\hat{W}$, we can estimate 
\begin{align*}
\hat{P}_t \equiv \hat{W} q_t
\end{align*}
starting from a state $q_t$. Since $p_t$ specifies a joint distribution over $e_{o_{t+1}}$ and $e_{o_t}$
we can easily condition on (or marginalize $o_t$) to obtain $q_{t+1}$. We will show that this is equivalent to
learning and applying observable operators as in \cite{hsu:09:hmm}: 

For a given value $x$ of $o_2$, define 
\begin{align}
B_x = u_x^\top \hat{W} = u_x^\top \hat{P}_{2:3,1} \left( \hat{P}_{2,1}^\top\right)^{-1},
\end{align}
where $u_x$ is an $|{\cal O}| \times |{\cal O}|^2$ matrix which selects a block of rows in
$\hat{P}_{2:3,1}$ corresponding to $o_2 = x$. 
Specifically, $u_x = \delta_x \otimes_k I_{|{\cal O}|}$. \footnote{Following the notation used in \cite{hsu:09:hmm}, 
$u_x^\top \hat{P}_{2:3,1} \equiv \hat{P}_{3,x,1}$}.

\begin{align*}
q_{t+1} & = \hat{\E}[e_{o_{t+1}} | o_{1:t}] \propto u_{o_{t}}^\top \hat{\E}[e_{o_{t:t+1}} | o_{1:t-1}]\\
& = u_{o_{t}}^\top \hat{\E}[\fstat_{t} | o_{1:t-1}] = u_{o_{t}}^\top \hat{W} \E[\psi_t | o_{1:t-1}] = B_{o_{t}} q_t
\end{align*}
with a normalization constant given by
\begin{align}
\frac{1}{1^\top B_{o_{t}} q_t}
\label{eq:hmm_normalizer}
\end{align}

Now we move to a more realistic setting, where we have $\mathrm{rank}(P_{2,1}) = m < |{\cal O}|$. 
Therefore we project the predictive state using a matrix $U$ that preserves the dynamics, by requiring that $U^\top O$ 
(i.e. $U$ is an independent set of columns spanning the range of the HMM observation matrix $O$).

It can be shown \cite{hsu:09:hmm} that ${\cal R}(O) = {\cal R}(P_{2,1}) = {\cal R}(P_{2,1} P_{1,1}^{-1})$.
Therefore, we can use the leading $m$ left singular vectors of $\hat{W}_{2,1}$
, which corresponds to replacing the linear regression in S1A with a reduced rank regression.
However, for the sake of our discussion we will use the singular vectors of $P_{2,1}$.
In more detail, let $[U,S,V]$ be the rank-$m$ SVD decomposition of $P_{2,1}$.
We use $\stat_t = U^\top e_{o_t}$
and $\fstat_t = e_{o_t} \otimes_k U^\top {e_{o_{t+1}}}$. 
S1 weights are then given by $\hat{W}^{rr}_{2,1} = U^\top \hat{W}_{2,1}$
and $\hat{W}^{rr}_{2:3,1} = (I_{|{\cal O}|} \otimes_k U^\top) \hat{W}_{2:3,1}$
and S2 weights are given by
\begin{align}
\nonumber \hat{W}^{rr} & = (I_{|{\cal O}|} \otimes_k U^\top) \hat{W}_{2:3,1} \hat{\E}[e_{o_1} e_{o_1}^\top] \hat{W}_{2,1}^\top U
\left(U^\top \hat{W}_{2,1} \hat{\E}[e_{o_1} e_{o_1}^\top] \hat{W}_{2,1}^\top U \right)^{-1} \\
\nonumber & = (I_{|{\cal O}|} \otimes_k U^\top) \hat{P}_{2:3,1} \hat{P}_{1,1}^{-1} V S \left(S V^\top \hat{P}_{1,1}^{-1} V S\right)^{-1}\\
 & = (I_{|{\cal O}|} \otimes_k U^\top) \hat{P}_{2:3,1} \hat{P}_{1,1}^{-1} V \left(V^\top \hat{P}_{1,1}^{-1} V\right)^{-1} S^{-1}
\label{eq:hmm_rr_w}
\end{align}

In the limit of infinite data, $V$ spans $\mathrm{range}(O) = \mathrm{rowspace}(P_{2:3,1})$ and hence 
$P_{2:3,1} = P_{2:3,1} VV^\top$. Substituting in \eqref{eq:hmm_rr_w} gives
\begin{align*}
W^{rr} = (I_{|{\cal O}|} \otimes_k U^\top) P_{2:3,1} V S^{-1} = (I_{|{\cal O}|} \otimes_k U^\top) P_{2:3,1} \left( U^\top P_{2,1} \right)^+
\end{align*}
Similar to the full-rank case we define, for each observation $x$ an $m \times |{\cal O}|^2$ selector matrix
$u_x = \delta_x \otimes_k I_m$ and an observation operator 
\begin{align}
B_x = u^\top_x \hat{W}^{rr} \to U^\top P_{3,x,1} \left( U^\top P_{2,1} \right)^+
\label{eq:hmm_rr_bx}
\end{align}
This is exactly the observation operator obtained in \cite{hsu:09:hmm}. 
However, instead of using \ref{eq:hmm_rr_w}, they use \ref{eq:hmm_rr_bx}
with $P_{3,x,1}$ and $P_{2,1}$ replaced by their empirical estimates.

Note that for a state $b_t = \E[\stat_t | o_{1:t-1}]$,  $B_x b_t = P(o_t | o_{1:t-1}) \E[\stat_{t+1} | o_{1:t}] = P(o_t | o_{1:t-1}) b_{t+1}$. 
To get $b_{t+1}$, the normalization constant becomes $\frac{1}{P(o_t | o_{1:t-1})} = \frac{1}{b_{\infty}^\top B_x b_t}$, 
where $b_\infty^\top b = 1$ for any valid predictive state $b$. To estimate $b_\infty$ we solve the aforementioned
condition for states estimated from all possible values of history features $h_t$. This gives,
\begin{align*}
b_\infty^\top \hat{W}^{rr}_{2,1} I_{|{\cal O}|} =
b_\infty^\top U^\top \hat{P}_{2,1} \hat{P}_{1,1}^{-1} I_{|{\cal O}|} = 1_{|{\cal O}|}^\top,
\end{align*}
where the columns of $I_{|{\cal O}|}$ represent all possible values of $h_t$. 
This in turn gives
\begin{align*}
b_\infty^\top & = 1_{|{\cal O}|}^\top \hat{P}_{1,1} (U^\top \hat{P}_{2,1})^+ \\
& = \hat{P}_1^\top (U^\top \hat{P}_{2,1})^+,
\end{align*}
the same estimator proposed in \cite{hsu:09:hmm}.

\subsection{Stationary Kalman Filter}
A Kalman filter is given by
\begin{align*}
s_t & = O s_{t-1} + \nu_t \\
o_t & = T s_t + \epsilon_t \\
\nu_t & \sim {\cal N}(0, \Sigma_s) \\
\epsilon_t & \sim {\cal N}(0, \Sigma_o) \\
\end{align*}
We consider the case of a \emph{stationary} filter where $\Sigma_t \equiv \E[s_t s_t^\top]$
is independent of $t$. 
We choose our statistics 
\begin{align*}
\pstat_t & = o_{t-H:t-1}\\
\stat_t & = o_{t:t+F-1}\\
\fstat_t & = o_{t:t+F},
\end{align*}

Where a window of observations is represented by stacking individual observations into a single vector. 
It can be shown \cite{boots:12:thesis,vanoverschee:96} that 
\begin{align*}
\E[s_t | \pstat_t] = \Sigma_{s,\pstat} \Sigma_{\pstat,\pstat}^{-1}\pstat_t
\end{align*}
and it follows that
\begin{align*}
\E [\stat_t | \pstat_t] & = \Gamma \Sigma_{s,\pstat} \Sigma_{\pstat,\pstat}^{-1}\pstat_t = W_1 h_t\\
\E [\fstat_{t} | \pstat_t] & = \Gamma_+ \Sigma_{s,\pstat} \Sigma_{\pstat,\pstat}^{-1}\pstat_t = W_2 h_t\\
\end{align*}
where $\Gamma$ is the extended observation operator
\begin{align*}
\Gamma \equiv \left(
\begin{array}{c}
O \\
OT \\
\vdots \\
OT^F
\end{array}
\right) , 
\Gamma_+ \equiv \left( \begin{array}{c}
O \\
OT \\
\vdots \\
OT^{F+1}
\end{array}
\right)
\end{align*}

It follows that $F$ and $H$ must be large enough to have $\mathrm{rank}(W) = n$. Let $U \in \mathbb{R}^{mF \times n}$ be the matrix of left singular values
of $W_1$ corresponding to non-zero singular values. 
Then $U^\top \Gamma$ is invertible and we can write
\begin{align*}
\E [\stat_t | \pstat_t] & = UU^\top \Gamma \Sigma_{s,\pstat} \Sigma_{\pstat,\pstat}^{-1}\pstat_t = W_1 h_t\\
\E [\fstat_{t} | \pstat_t] & = \Gamma_+ \Sigma_{s,\pstat} \Sigma_{\pstat,\pstat}^{-1}\pstat_t = W_2 h_t\\
\E [\fstat_{t} | \pstat_t] & = \Gamma_+ (U^\top \Gamma)^{-1} U^\top \left(UU^\top \Gamma \Sigma_{s,\pstat} \Sigma_{\pstat,\pstat}^{-1}\pstat_t\right) \\
& = W \E [\stat_t | \pstat_t]\\
\end{align*}
which matches the instrumental regression framework. 
\todo{Covariance pseudo-inverse}
For the steady-state case (constant Kalman gain), one can estimate
$\Sigma_{\fstat}$ given the data and the parameter $W$ by solving
Riccati equation as described in \cite{vanoverschee:96}.
$\E [\fstat_{t} | o_{1:t-1}]$ and $\Sigma_{\fstat}$ then specify a joint Gaussian distribution
over the next $F+1$ observations where marginalization and conditioning can be easily performed.

We can also assume a Kalman filter that is not in the steady state (i.e. the Kalman gain is not constant).
In this case we need to maintain sufficient statistics for a predictive Gaussian distribution
(i.e. mean and covariance). 
Let $\mathrm{vec}$ denote the vectorization operation, which stacks the columns of a matrix
into a single vector. We can stack $h_t$ and $\mathrm{vec}(h_t h_t^\top)$ to
into a single vector that we refer to as 1st+2nd moments vector. We do the same for future and extended future.
We can, in principle, perform linear regression on these 1st+2nd moment vectors but that 
requires an unnecessarily large number of parameters. Instead, we can learn an S1A regression function of the form
\begin{align}
\E[\stat_t | h_t] & = W_1 h_t \\
\E[\stat_t \stat_t^\top | h_t] & = W_1 h_t h_t^\top W_1 + R \\
\end{align}
Where $R$ is simply the covariance of the residuals of the 1st moment regression
(i.e. covariance of $r_t = \stat_t - \E[\stat_t | h_t]$). This is still a linear model
in terms of 1st+2nd moment vectors and hence we can do the same for S1B and S2 regression models. This way, the extended belief vector $p_t$ (the expectation of 1st+2nd moments of extended future) fully specifies a joint distribution
over the next $F+1$ observations.

\subsection{HSE-PSR}
We define a class of non-parametric two-stage instrumental regression models. 
By using conditional mean embedding \cite{song:09} as S1 regression model, we recover 
a single-action variant of HSE-PSR \cite{boots:13:hsepsr}.
Let $\h{X}, \h{Y}, \h{Z}$ denote three reproducing kernel Hilbert spaces with
reproducing kernels $\kernel{X}, \kernel{Y}$ and $\kernel{Z}$ respectively.
Assume $\stat_t \in \h{X}$ 
and that $\fstat_t \in {\cal Y}$ is defined as the tuple $(o_{t} \otimes o_{t}, \stat_{t+1} \otimes o_{t})$.
Let $\stattrain \in \h{X} \otimes \mathbb{R}^N$, 
$\fstattrain \in \h{Y} \otimes \mathbb{R}^N$ 
and $\pstattrain \in \h{Z} \otimes \mathbb{R}^N$
be operators that represent training data.
Specifically, $\stat_s$, $\fstat_s$, $\pstat_s$ are the $s^{th}$ "columns" in $\stattrain$ and $\fstattrain$ and $\pstattrain$ respectively.
It is possible to implement S1 using a non-parametric regression method
that takes the form of a linear smoother. In such case the training data for S2 regression take the form
\begin{align*}
\hat{\E}[\stat_t \mid \pstat_t] & = \sum_{s=1}^N \beta_{s \mid h_t} \psi_s \\
\hat{\E}[\fstat_t \mid \pstat_t] & = \sum_{s=1}^N \gamma_{s \mid h_t} \fstat_s,
\end{align*}
where $\beta_s$ and $\gamma_s$ depend on $\pstat_t$. This produces the following training operators 
for S2 regression:
\begin{align*}
\tilde{\stattrain} & = \stattrain \mathbf{B} \\
\tilde{\fstattrain} & = \fstattrain \mathbf{\Gamma},
\end{align*}
where $\mathbf{B}_{st} = \beta_{s|h_t}$ and 
$\mathbf{\Gamma}_{st} = \gamma_{s|h_t}$.
With this data, S2 regression 
uses a Gram matrix formulation to estimate the operator
\begin{align}
W = \fstattrain \mathbf{\Gamma} (\mathbf{B}^\top G_{\h{X}, \h{X}}  \mathbf{B} + \lambda I_N)^{-1} \mathbf{B}^\top \stattrain^*
\label{eq:rkhs_s2w}
\end{align}

Note that we can use an arbitrary method to estimate $\mathbf{B}$. Using conditional mean maps,
the weight matrix $\mathbf{B}$
is computed using kernel ridge regression
\begin{align}
\mathbf{B} = (G_{\h{Z},\h{Z}} + \lambda I_N)^{-1} G_{\h{Z},\h{Z}}
\label{eq:vectorreg}
\end{align}

HSE-PSR learning is similar to this setting, with $\stat_t$
being a conditional expectation operator of test observations given test actions.
For this reason, kernel ridge regression is replaced by application of kernel Bayes rule
\cite{fukumizu:13:kbr}.

For each $t$, S1 regression will produce a denoised prediction
$\hat{E}[\fstat_t \mid h_t]$ as a linear combination of training feature maps
\begin{align*}
\hat{E}[\fstat_t \mid h_t]& = \fstattrain \alpha_{t} = \sum_{s=1}^N \alpha_{t,s} \fstat_s \\
\end{align*}
This corresponds to the covariance operators
\begin{align*}
\hat{\Sigma}_{\stat_{t+1} o_{t} \mid h_t} & = \sum_{s=1}^N \alpha_{t,s} \stat_{s+1} \otimes o_{s}
= \stattrain' \mathrm{diag}(\alpha_t) \obstrain^*
 \\
\hat{\Sigma}_{o_{t} o_{t}  \mid h_t} & = \sum_{s=1}^N \alpha_{t,s} o_{s} \otimes o_{s}
= \obstrain \mathrm{diag}(\alpha_t) \obstrain^*
\end{align*}
Where, $\stattrain'$ is the shifted future training operator satisfying $\stattrain' e_t = \psi_{t+1}$  Given these two covariance operators, we can 
use kernel Bayes rule \cite{fukumizu:13:kbr} to condition on $o_{t}$ which gives
\begin{align}
q_{t+1} = \hat{E}[\psi_{t+1}  \mid h_t] = \hat{\Sigma}_{\stat_{t+1} o_{t}  \mid h_t}
(\hat{\Sigma}_{o_{t} o_{t}  \mid h_t} + \lambda I)^{-1} o_{t}.
\label{eq:rkhs_cond}
\end{align}
Replacing $o_{t}$ in \eqref{eq:rkhs_cond} with its conditional expectation
$\sum_{s=1}^N \alpha_s o_s$ corresponds to marginalizing over $o_t$ (i.e. prediction).
A stable Gram matrix formulation for \eqref{eq:rkhs_cond} is given by \cite{fukumizu:13:kbr}
\begin{align}
\nonumber & q_{t+1} \\ 
\nonumber & \quad = \stattrain' \mathrm{diag}(\alpha_t) G_{{\cal O},{\cal O}} 
((\mathrm{diag}(\alpha_t) G_{{\cal O},{\cal O}})^2 + \lambda N I)^{-1} \\
\nonumber & \quad\quad .\mathrm{diag}(\alpha_t) \obstrain^* o_{t+1} \\
& \quad = \stattrain' \tilde{\alpha}_{t+1},
\label{eq:rkhs_cond_gram}
\end{align}
which is the state update equation in HSE-PSR.
Given $\tilde{\alpha}_{t+1}$ we perform S2 regression to estimate 
\begin{align*}
\hat{P}_{t+1} = \hat{\E}[\fstat_{t+1} \mid o_{1:t+1}] = \fstattrain \alpha_{t+1}
= W \stattrain' \tilde{\alpha}_{t+1},
\end{align*}
where $W$ is defined in \eqref{eq:rkhs_s2w}.
\section{Proofs}
\subsection{Proof of Main Theorem}
In this section we provide a proof for theorem \ref{thm:main_short}.
We provide finite sample analysis of the effects of S1 regression, covariance estimation and regularization.
The asymptotic statement becomes a natural consequence. 

We will make use of matrix Bernstein's inequality
stated below:
\begin{lemma}[Matrix Bernstein's Inequality \cite{hsu:12b:matrix}]
Let $A$ be a random square symmetric matrix, and $r > 0$, $v > 0$ and $k > 0$ be 
such that, almost surely,
\begin{align*}
\E[A] = 0, \quad \lambda_{\max}[A] \leq r, \\
\lambda_{\max}[\E[A^2]] \leq v, \quad \tr(\E[A^2]) \leq k.
\end{align*}
If $A_1,A_2,\dots,A_N$ are independent copies of $A$, then for any $t > 0$,
\begin{align}
\nonumber \Pr\left[\lambda_{\max} \left[ \frac{1}{N} \sum_{t=1}^N A_t \right] 
> \sqrt{\frac{2vt}{N}} + \frac{rt}{3N} \right] \\
\quad\quad \leq \frac{kt}{v}(e^t - t - 1)^{-1}.
\label{eq:brenestein1}
\end{align}
If $t \geq 2.6$, then $t(e^t - t - 1)^{-1} \leq e^{-t/2}.$
\label{thm:brenestein1}
\end{lemma}

Recall that, assuming $x_{test} \in {\cal R}(\Cov{x})$, we have three sources of error:
first, the error in S1 regression causes the input to S2 regression procedure $(\hat{x}_t, \hat{y}_t)$ to be a perturbed
version of the true $(\bar{x}_t, \bar{y}_t)$;
second, the covariance operators are estimated from a finite sample of size $N$; 
and third, there is the effect of regularization. In the proof, we characterize 
the effect of each source of error. To do so, we define the following intermediate quantities:
\todo{Define $\CCov{y}{x}$ clearly in terms of x,y,z (in the main text?)}
\begin{align}
W_\lambda = \CCov{y}{x} \left( \Cov{x} + \lambda I \right)^{-1} \label{eq:wlambda}\\
\bar{W}_\lambda = \ACCov{y}{x} \left( \ACov{x} + \lambda I \right)^{-1} \label{eq:wbar},
\end{align}
where
\begin{align*}
\ACCov{y}{x} \equiv \frac{1}{N} \sum_{t=1}^N \bar{y}_t \otimes \bar{x}_t
\end{align*}
and $\ACov{x}$ is defined similarly.
Basically, $W_\lambda$ captures only the effect of regularization and $\bar{W}_\lambda$
captures in addition the effect of finite sample estimate of the covariance.
$\bar{W}_\lambda$ is the result of S2 regression if $\bar{x}$ and $\bar{y}$
were perfectly recovered by S1 regression. 
It is important to note that $\ACCov{x}{y}$ and $\ACov{x}$
are \emph{not} observable quantities since they depend on 
the true expectations $\bar{x}$ and $\bar{y}$. 
We will use $\lambdax{i}$ and $\lambday{i}$ to denote the $i^{th}$ eigenvalue of $\Cov{x}$ and $\Cov{y}$
respectively in descending order and we will use $\|.\|$ to denote the operator norm.

Before we prove the main theorem, we define the quantities $\ecovxx$ and $\ecovxy$ which we use 
to bound the effect of covariance estimation from finite data, as stated in the following lemma:
\begin{lemma}[Covariance error bound]
Let $N$ be a positive integer and $\delta \in (0,1)$ and assume that $\|\bar{x}\|, \|\bar{y}\| < c < \infty$ almost surely.
Let $\ecovxy$ be defined as:
\begin{align}
\ecovxy = \sqrt{\frac{2vt}{N}} + \frac{rt}{3N},
\label{eq:ecovxy}
\end{align}
where
\begin{align*}
t & = \max(2.6, 2 \log(4k/\delta v)) \\
r & = c^2 + \|\CCov{y}{x}\| \\
v & = c^2 \max(\lambday{1}, \lambdax{1}) + \|\CCov{x}{y}\|^2 \\
k & = c^2 (\tr(\Cov{x}) + \tr(\Cov{y}))
\end{align*}
In addition, let $\ecovxx$ be defined as:
\begin{align}
\ecovxx = \sqrt{\frac{2v't'}{N}} + \frac{r't'}{3N},
\label{eq:ecovxx}
\end{align}
where
\begin{align*}
t' & = \max(2.6, 2 \log(4k'/\delta v')) \\
r' & = c^2 + \lambdax{1} \\
v' & = c^2 \lambdax{1} + \lambdax{1}^2 \\
k' & = c^2 \tr(\Cov{x})
\end{align*}
and define $\ecovyy$ similarly for $\Cov{y}$. 

It follows that, with probability at least $1 - \delta/2$,
\begin{align*}
\|\ACCov{y}{x} - \CCov{y}{x} \| < \ecovxy \\
\|\ACov{x} - \Cov{x} \| < \ecovxx \\
\|\ACov{y} - \Cov{y} \| < \ecovyy \\
\end{align*}
\label{thm:cov_bound}
\end{lemma}
\begin{proof}
We show that each statement holds with probability at least $1 - \delta/6$. 
The claim then follows directly from the union bound.
We start with $\ecovxx$. By setting $A_t = \bar{x}_t \otimes \bar{x}_t - \Cov{x}$ 
then we would like to obtain a high probability bound on $\|\frac{1}{N} \sum_{t=1}^N A_t\|$.
Lemma \ref{thm:brenestein1} shows that, in order to satisfy the bound with probability
at least $1 - \delta/6$, it suffices to set
$t$ to $\max(2.6, 2k \log(6/\delta v))$. So, it remains to 
find suitable values for $r, v$ and $k$:
\begin{align*}
\lambda_{\max}[A] & \leq \|\bar{x}\|^2 + \|\Cov{x}\| \leq c^2 + \lambdax{1} = r' \\
\lambda_{\max}[\E[A^2]] & = \lambda_{\max}[\E[\| \bar{x} \|^2 (\bar{x} \otimes \bar{x}) 
 - (\bar{x} \otimes \bar{x}) \Cov{x}
- \Cov{x}(\bar{x} \otimes \bar{x}) + \Cov{x}^2] \\
  & = \lambda_{\max}[\E[\| \bar{x} \|^2 (\bar{x} \otimes \bar{x}) - \Cov{x}^2]] 
 \leq c^2 \lambdax{1} + \lambdax{1}^2 = v'\\
\tr[\E[A^2]] & = \tr[\E[\| \bar{x} \|^2 (\bar{x} \otimes \bar{x}) - \Cov{x}^2]] 
\leq \tr[\E[\| \bar{x} \|^2 (\bar{x} \otimes \bar{x})]] \leq c^2 \tr(\Cov{x}) = k'
\end{align*}

The case of $\ecovyy$ can be proven similarly. Now moving to $\ecovxy$, we have $B_t = \bar{y}_t \otimes \bar{x}_t - \CCov{y}{x}$. Since $B_t$ is not square, we use the Hermitian dilation ${\mathscr H}(B)$ 
 defined as follows\cite{tropp:12}:
\begin{align*}
A = {\mathscr H}(B)
 = \left[
\begin{array}{cc}
0 & B \\
B^* & 0
\end{array}
\right]
\end{align*}
Note that
\begin{align*}
\lambda_{\max}[A] = \|B\|,
\quad
A^2 = \left[
\begin{array}{cc}
BB^* & 0 \\
0 & B^*B
\end{array}
\right]
\end{align*}
therefore suffices to bound $\|\frac{1}{N} \sum_{t=1}^N A_t\|$
using an argument similar to that used in $\ecovxx$ case. 
\end{proof}

To prove theorem \ref{thm:main_short}, we write
\begin{align}
\nonumber \| \hat{W}_{\lambda} x_{\rm test} - W x_{\rm test} \|_\h{Y} & \leq \|(\hat{W}_\lambda - \bar{W}_\lambda) \bar{x}_{\rm test}\|_\h{Y} \\
\nonumber & + \|(\bar{W}_\lambda - {W}_\lambda) \bar{x}_{\rm test}\|_\h{Y} \\
& + \|(W_\lambda - W) \bar{x}_{\rm test}\|_\h{Y}
\label{eq:proof_plan}
\end{align}
We will now present bounds on each term. We consider the case where $\bar{x}_{\rm test} \in \srange{x}$.
Extension to $\scrange{x}$ is a result of the assumed boundedness of $W$, which implies the boundedness of $\hat{W}_\lambda - W$.
%
%
\begin{lemma}[Error due to S1 Regression]
Assume that $\|\bar{x}\|, \|\bar{y}\| < c < \infty$ almost surely,
and let $\efs$ be as defined in Definition \ref{thm:efs}. The following holds with probability at least 
$1 - \delta$
\begin{align*}
\|\hat{W}_\lambda - \bar{W}_\lambda\|
& \leq \sqrt{\lambday{1} + \ecovyy}\frac{(2c \efs + {\efs}^2)}{\lambda^{\frac{3}{2}}} \\
& + \frac{(2c \efs + {\efs}^2)}{\lambda} \\
& = O\left(\efs \left(\frac{1}{\lambda} + \frac{\sqrt{1 + \frac{\log(1/\delta)}{\sqrt{N}}}}{\lambda^{\frac{3}{2}}} \right)\right).
\end{align*}
The asymptotic statement assumes $\efs \to 0$ as $N \to \infty$.
\label{thm:err_s1}
\end{lemma}

\begin{proof}
Write $\AACov{x} = \ACov{x} + \Delta_{x}$ and $\AACCov{y}{x} = \ACov{y}{x} + \Delta_{yx}$.
We know that, with probability at least $1 - \delta/2$, the following 
is satisfied for all unit vectors $\phi_x \in \h{X}$ and $\phi_y \in \h{Y}$
\begin{align*}
\innerh{\phi_y}{\Delta_{yx} \phi_x}{Y} &= \frac{1}{N} \sum_{t=1}^N 
\innerh{\phi_y}{\hat{y}_t}{Y} \innerh{\phi_x}{\hat{x}_t}{X}\\
& - \innerh{\phi_y}{\hat{y}_t}{Y} \innerh{\phi_x}{\bar{x}_t}{X} \\
& + \innerh{\phi_y}{\hat{y}_t}{Y} \innerh{\phi_x}{\bar{x}_t}{X}
- \innerh{\phi_y}{\bar{y}_t}{Y} \innerh{\phi_x}{\bar{x}_t}{X} \\
& = \frac{1}{N} \sum_t \innerh{\phi_y}{\bar{y}_t + (\hat{y}_t - \bar{y}_t)}{Y} \innerh{\phi_x}{\hat{x}_t - \bar{x}_t}{X}\\
& + \innerh{\phi_y}{\hat{y}_t - \bar{y}_t}{Y} \innerh{\phi_x}{\bar{x}_t}{X} \\
& \leq 2c \efs + \efs^2
\end{align*}
Therefore,
\begin{align*}
\|\Delta_{yx}\| = \sup_{\|\phi_x\|_\h{X} \leq 1,\|\phi_y\|_\h{Y} \leq 1} \innerh{\phi_y}{\Delta_{yx} \phi_x}{Y} 
\leq 2c \efs + \efs^2,
\end{align*}
and similarly 
\begin{align*}
\|\Delta_{x}\| \leq 2c \efs + {\efs}^2,
\end{align*}
with probability $1 - \delta/2$. We can write 
\begin{align*}
\hat{W}_\lambda - \bar{W}_\lambda & = \ACCov{y}{x}\left((\ACov{x} + \Delta_{x} + \lambda I)^{-1} - (\ACov{x} + \lambda I)^{-1} \right) \\
& + \Delta_{yx}(\ACov{x} + \Delta_{x} + \lambda I)^{-1}
\end{align*}
Using the fact that $B^{-1} - A^{-1} = B^{-1}(A - B)A^{-1}$ for invertible operators $A$ and $B$ we get
\begin{align*}
\hat{W}_\lambda - \bar{W}_\lambda & = -\ACCov{y}{x}(\ACov{x} + \lambda I)^{-1} \Delta_x (\ACov{x} + \Delta_{x} + \lambda I)^{-1} \\
& + \Delta_{yx}(\ACov{x} + \Delta_{x} + \lambda I)^{-1}
\end{align*}
we then use the decomposition $\ACCov{y}{x} = \ACov{y}^\frac{1}{2} V \ACov{x}^\frac{1}{2}$,
where $V$ is a correlation operator satisfying $\|V\| \leq 1$. This gives
\begin{align*}
& \hat{W}_\lambda - \bar{W}_\lambda = \\
& \quad -\ACov{y}^\frac{1}{2} V \ACov{x}^\frac{1}{2}(\ACov{x} + \lambda I)^{-\frac{1}{2}}(\ACov{x} + \lambda I)^{-\frac{1}{2}} \Delta_x (\ACov{x} + \Delta_{x} + \lambda I)^{-1} \\
& \quad + \Delta_{yx}(\ACov{x} + \Delta_{x} + \lambda I)^{-1}
\end{align*}
Noting that $\|\ACov{x}^{\frac{1}{2}}(\ACov{x} + \lambda I)^{-\frac{1}{2}}\| \leq 1$,
the rest of the proof follows from triangular inequality and the fact that $\|AB\| \leq \|A\| \|B\|$
\end{proof}
\begin{lemma}[Error due to Covariance]
Assuming that $\|\bar{x}\|_{\cal X}, \|\bar{y}\|_{\cal Y} < c < \infty$ almost surely, the following holds
with probability at least $1-\frac{\delta}{2}$
\begin{align*}
\| \bar{W}_\lambda - W_\lambda \| & \leq \sqrt{\lambday{1}} \ecovxx \lambda^{-\frac{3}{2}} + \frac{\ecovxy}{\lambda}
\end{align*},
where $\ecovxx$ and $\ecovxy$ are as defined in Lemma \ref{thm:cov_bound}.
\label{thm:err_cov}
\end{lemma}
\begin{proof}
Write $\ACov{x} = \Cov{x} + \Delta_{x}$ and $\ACCov{y}{x} = \CCov{y}{x} + \Delta_{yx}$.
Then we get
\begin{align*}
\bar{W}_\lambda - W_\lambda & = \CCov{y}{x}\left((\Cov{x} + \Delta_{x} + \lambda I)^{-1} - (\Cov{x} + \lambda I)^{-1} \right)
+ \Delta_{yx}(\Cov{x} + \Delta_{x} + \lambda I)^{-1}
\end{align*}
Using the fact that $B^{-1} - A^{-1} = B^{-1}(A - B)A^{-1}$ for invertible operators $A$ and $B$ we get
\begin{align*}
\bar{W}_\lambda - W_\lambda & = -\CCov{y}{x}(\Cov{x} + \lambda I)^{-1} \Delta_x (\Cov{x} + \Delta_{x} + \lambda I)^{-1} 
+ \Delta_{yx}(\Cov{x} + \Delta_{x} + \lambda I)^{-1}
\end{align*}
we then use the decomposition $\CCov{y}{x} = \Cov{y}^\frac{1}{2} V \Cov{x}^\frac{1}{2}$,
where $V$ is a correlation operator satisfying $\|V\| \leq 1$. This gives
\begin{align*}
& \bar{W}_\lambda - W_\lambda = \\
& \quad -\Cov{y}^\frac{1}{2} V \Cov{x}^\frac{1}{2}(\Cov{x} + \lambda I)^{-\frac{1}{2}}(\Cov{x} + \lambda I)^{-\frac{1}{2}} \\
& \quad .\Delta_x (\Cov{x} + \Delta_{x} + \lambda I)^{-1} \\
& \quad + \Delta_{yx}(\Cov{x} + \Delta_{x} + \lambda I)^{-1}
\end{align*}
Noting that $\|\Cov{x}^{\frac{1}{2}}(\Cov{x} + \lambda I)^{-\frac{1}{2}}\| \leq 1$,
the rest of the proof follows from triangular inequality and the fact that $\|AB\| \leq \|A\| \|B\|$
\end{proof}
\begin{lemma}[Error due to Regularization on inputs within ${\cal R}(\Cov{x})$]
For any $x \in \srange{x}$ s.t. $\|x\|_\h{X} \leq 1$ and $\|\Cov{x}^{-\frac{1}{2}} x\|_\h{X} \leq C$. The following holds
\begin{align*}
\| (W_\lambda - W)x \|_\h{Y} \leq \frac{1}{2}\sqrt{\lambda} \|W\|_{HS} C
\end{align*}
\label{thm:err_reg}
\end{lemma}
\begin{proof}
Since $x \in \srange{x} \subseteq {\cal R}(\Cov{x}^\frac{1}{2})$,
\todo{Is this true ?}
we can write $x = \Cov{x}^\frac{1}{2} v$ for some $v \in \h{X}$
s.t. $\|v\|_\h{X} \leq C$.
Then 
\begin{align*}
(W_\lambda - W)x = \CCov{y}{x}((\Cov{x} + \lambda I)^{-1} - \Cov{x}^{-1})\Cov{x}^\frac{1}{2} v
\end{align*}

Let $D = \CCov{y}{x}((\Cov{x} + \lambda I)^{-1} - \Cov{x}^{-1})\Cov{x}^\frac{1}{2}$.
We will bound the Hilbert-Schmidt norm of $D$. Let $\ux{i} \in \h{X}$, $\uy{i} \in \h{Y}$ denote the eigenvector
corresponding to $\lambdax{i}$ and $\lambday{i}$ respectively. Define $s_{ij} = | \innerh{\uy{j}}{\CCov{x}{y} \ux{i}}{Y} |$.
Then we have
\begin{align*}
| \innerh{\uy{j}}{D \ux{i}}{Y} | & = \left|\innerh{\uy{j}}{\CCov{y}{x} \frac{\lambda}{(\lambdax{i} + \lambda)\sqrt{\lambdax{i}}} \ux{i}}{Y} \right| \\
& = \frac{\lambda s_{ij}}{(\lambdax{i} + \lambda)\sqrt{\lambdax{i}}} = \frac{s_{ij}}{\sqrt{\lambdax{i}}} \frac{1}{\frac{1}{\lambda/\lambdax{i}} + 1} \\
& \leq \frac{s_{ij}}{\sqrt{\lambdax{i}}}.\frac{1}{2} \sqrt{\frac{\lambda}{\lambdax{i}}} = \frac{1}{2} \sqrt{\lambda} \frac{s_{ij}}{\lambdax{i}} \\
& = \frac{1}{2} \sqrt{\lambda} | \innerh{\uy{j}}{W \ux{i}}{Y} |,
\end{align*}
where the inequality follows from the arithmetic-geometric-harmonic mean inequality. This gives the following bound
\begin{align*}
\|D\|_{HS}^2 = \sum_{i,j} \innerh{\uy{j}}{D \ux{i}}{Y}^2 \leq \frac{1}{2}\sqrt{\lambda} \|W\|_{HS}^2
\end{align*}
and hence
\begin{align*}
\| (W_\lambda - W)x \|_\h{Y} & \leq \|D\| \|v\|_\h{X} \leq \|D\|_{HS} \|v\|_\h{X} \\
& \leq \frac{1}{2}\sqrt{\lambda} \|W\|_{HS} C 
\end{align*}
\end{proof}

Note that the additional assumption that $\|\Cov{x}^{-\frac{1}{2}} x\|_\h{X} \leq C$ is not required to obtain an asymptotic $O(\sqrt{\lambda})$ rate
for a given $x$.
This assumption, however, allows us to uniformly bound the constant.
Theorem \ref{thm:main_short} is simply the result of plugging the bounds in Lemmata \ref{thm:err_s1}, \ref{thm:err_cov}, 
and \ref{thm:err_reg} into \eqref{eq:proof_plan} and using the union bound.

\subsection{Proof of Lemma \ref{thm:orth_state_bound}}
for $t=1$:
Let ${\cal I}$ be an index set over training instances such that
\begin{align*}
\hat{Q}^{\rm test}_1 = \frac{1}{|{\cal I}|} \sum_{i \in {\cal I}} \hat{Q}_i
\end{align*}
Then
\begin{align*}
\| \hat{Q}^{\rm test}_1 - \tilde{Q}^{\rm test}_1 \|_\h{X} = \frac{1}{|{\cal I}|} \sum_{i \in {\cal I}} \| \hat{Q}_i - \tilde{Q}_i \|_\h{X} 
\leq \frac{1}{|{\cal I}|} \sum_{i \in {\cal I}} \| \hat{Q}_i - Q_i \|_\h{X} \leq \efs
\end{align*}

for $t > 1$:
Let $A$ denote a projection operator on $\orange{y}$
\begin{align*}
& \| \hat{Q}^{\rm test}_{t+1} - \tilde{Q}^{\rm test}_{t+1} \|_\h{X} \leq L \| \hat{P}^{\rm test}_t - \tilde{P}^{\rm test}_t\|_\h{Y}
\leq L \| A \hat{W}_\lambda \hat{Q}^{\rm test}_t \|_\h{Y} \\
& \quad \leq L \left\| \frac{1}{N} \left( \sum_{i=1}^N A \hat{P}_i \otimes \hat{Q}_i \right) \left( \frac{1}{N} \sum_{i=1}^N \hat{Q}_i \otimes \hat{Q}_i + \lambda I \right)^{-1}  \right\| \left\| \hat{Q}^{\rm test}_t \right\|_\h{X} \\
& \quad \leq L \left\| \frac{1}{N} \sum_{i=1}^N A \hat{P}_i \otimes A \hat{P}_i \right\|^\frac{1}{2} \frac{1}{\sqrt{\lambda}} \| \hat{Q}^{\rm test}_t \|_\h{X} 
\leq L \frac{\efs}{\sqrt{\lambda}} \| \hat{Q}^{\rm test}_t \|_\h{X},
\end{align*}
where the second to last inequality follows from the decomposition similar to $\Sigma_{YX} = \Sigma_{Y}^\frac{1}{2} V \Sigma_{X}^\frac{1}{2}$,
and the last inequality follows from the fact that $\| A \hat{P}_i \|_\h{Y} \leq \| \hat{P}_i - \bar{P}_i \|_\h{Y}$. 
\hfill$\qed$

\section{Examples of S1 Regression Bounds}
The following propositions provide concrete examples of S1 regression bounds $\efs$ for practical 
regression models.

\begin{proposition}
Assume $\h{X} \equiv \mathbb{R}^{d_x}, \mathbb{R}^{d_y}, \mathbb{R}^{d_z}$ for some
$d_x,d_y,d_z < \infty$ and that $\bar{x}$ and $\bar{y}$ are linear vector functions of $z$
where the parameters are estimated using ordinary least squares.
Assume that $\|\bar{x}\|_{\h{X}}, \|\bar{y}\|_{\h{Y}} < c < \infty$
almost surely.
Let $\efs$ be as defined in Definition \ref{thm:efs}.
Then
\begin{align*}
\efs = O \left(\sqrt{\frac{d_z}{N}} \log((d_x + d_y)/\delta) \right)
\end{align*}
\label{thm:ols}
\end{proposition}
\begin{proof}(sketch)
This is based on results that bound parameter estimation
error in linear regression with univariate response (e.g. \cite{hsu:12}).
Note that if $\bar{x}_{ti}= U_i^\top z_t$ for some $U_i \in \h{Z}$,
then a bound on the error norm
$\| \hat{U_i} - U_i \|$ implies a uniform bound of the same rate on 
$\hat{x_i} - \bar{x}$.
The probability of exceeding the bound is scaled by $1/(d_x + d_y)$
to correct for multiple regressions.
\end{proof}
Variants of Proposition \ref{thm:ols} can also be developed using bounds on non-linear regression models 
(e.g., generalized linear models).

The next proposition addresses a scenario where $\h{X}$ and $\h{Y}$ are infinite dimensional.
\begin{proposition}
Assume that $x$ and $y$ are kernel evaluation functionals, $\bar{x}$ and $\bar{y}$ are linear vector functions of $z$
where the linear operator is estimated using conditional mean embedding \cite{song:09}
with regularization parameter $\lambda_0 > 0$
and that $\|\bar{x}\|_{\h{X}}, \|\bar{y}\|_{\h{Y}} < c < \infty$ almost surely.
Let $\efs$ be as defined in Definition \ref{thm:efs}.
It follows that
\begin{align*}
\efs = O \left(\sqrt{\lambda_0} + \sqrt{\frac{\log(N/\delta)}{\lambda_0 N}} \right)
\end{align*}
\end{proposition}
\begin{proof}(sketch)
This bound is based on \cite{song:09}, which gives a bound on the error in estimating the conditional mean embedding.
The error probability is adjusted by $\delta/4N$ to accommodate the requirement that the bound holds for all training data.
\end{proof}

\end{document}